\documentclass[11pt]{article}
\usepackage[utf8x]{inputenc}
\usepackage[a4paper, left=3cm, right=3cm, top=3cm, bottom=3cm]{geometry}%
\usepackage{times}
\usepackage[hidelinks]{hyperref}       
\usepackage{url}            
\usepackage{booktabs}       
\usepackage{amsfonts}       
\usepackage{nicefrac}       
\usepackage{microtype}      
\usepackage[usenames,dvipsnames]{xcolor}         
\usepackage{tikz}

\usepackage{mathrsfs}

\usepackage{algorithmic}
\usepackage{algorithm}

\usepackage{amsmath,amsfonts,amssymb}
\usepackage{amsthm}

\usepackage{mathtools}
\usepackage{rotating}
\usepackage{subcaption}
\usepackage{soul}
\usepackage{natbib}
\usepackage{float}

\renewcommand{\leq}{\leqslant}
\renewcommand{\geq}{\geqslant}

\newcommand{\eps}{\varepsilon}

\newcommand{\argmin}{\mathop{\mathrm{argmin}}}

\newcommand{\wh}{\widehat}
\newcommand{\pp}{\; : \; }

\newcommand{\bs}{\boldsymbol}

\newcommand{\R}{\mathbb R}

\newcommand{\cX}{\mathcal X}
\newcommand{\cY}{\mathcal Y}

\renewcommand{\d}{\mathrm{d}}

\newcommand{\otb}{\ensuremath{\mathtt{oob}}\xspace}
\newcommand{\itb}{\ensuremath{\mathtt{itb}}\xspace}

\usepackage{xspace}

\renewcommand{\P}{\mathbb P}

\newcommand{\probas}{\mathcal{P}}
\newcommand{\kl}{\mathrm{KL}} 

\newcommand{\dirichletdist}{\mathsf{Dir}}

\newcommand{\loss}{\ell}

\newcommand{\pred}{\widehat{y}}

\newcommand{\splits}{\Sigma} 
\newcommand{\node}{\mathbf{v}} 

\newcommand{\nodes}{\mathrm{nodes}} 
\newcommand{\inodes}{\mathrm{intnodes}} 
\newcommand{\leaves}{\mathrm{leaves}} 

\newcommand{\leaf}{\node}

\newcommand{\cell}{C}

\renewcommand{\root}{\mathtt{root}} 

\newcommand{\tree}{\mathcal{T}} 

\newcommand{\pathpoint}{\mathtt{path}} 
\newcommand{\wbar}{w^{\mathrm{den}}} 
\newcommand{\wnum}{w^{\mathrm{num}}} 

\newcommand{\wpred}{\widehat{w}}

\newcommand{\bestb}[1]{\textbf{#1}}
\newcommand{\bestf}[1]{\underline{#1}}

\renewcommand{\widehat}{\hat}

\newcommand{\ind}[1]{\mathbf 1_{#1}}

\newtheorem{theorem}{Theorem}
\newtheorem{lemma}{Lemma}
\newtheorem{corollary}{Corollary}

\usepackage{listings}
\usepackage{todonotes}

\title{WildWood: a new Random Forest algorithm}

\author{%
  St\'ephane Ga\"iffas\thanks{\texttt{stephane.gaiffas@lpsm.paris}} \\
  LPSM, Université Paris Cité \\
  DMA, \'Ecole normale sup\'erieure\\
  \and
  Ibrahim Merad\thanks{\texttt{imerad@lpsm.paris}}\\
  LPSM, Université Paris Cité
  \and
  Yiyang Yu\thanks{\texttt{yyu@lpsm.paris}}\\
  LPSM, Université Paris Cité
}

\begin{document}

\maketitle

\begin{abstract}
    We introduce WildWood (WW), a new ensemble algorithm for supervised learning of Random Forest (RF) type. 
    While standard RF algorithms use bootstrap out-of-bag samples to compute out-of-bag scores, WW uses these samples to produce improved predictions given by an aggregation of the predictions of all possible subtrees of each fully grown tree in the forest.
    This is achieved by aggregation with exponential weights computed over out-of-bag samples, that are computed exactly and very efficiently thanks to an algorithm called context tree weighting.
    This improvement, combined with a histogram strategy to accelerate split finding, makes WW fast and competitive compared with other well-established ensemble methods, such as standard RF and extreme gradient boosting algorithms.
\end{abstract}

\section{Introduction}
\label{sec:introduction}

This paper introduces WildWood (WW), a new ensemble method of Random Forest (RF) type~\citep{breiman2001randomforests}.
The main contributions of the paper and the main advantages of WW are as follows.
Firstly, we use out-of-bag samples (trees in a RF use different bootstrapped samples) very differently than what is done in standard RF~\citep{louppe2014understanding, biau2016rf_tour}. 
Indeed, WW uses these samples to compute an aggregation of the predictions of all possible subtrees of each tree in the forest, using aggregation with exponential weights~\citep{catoni2004statistical}.
This leads to much improved predictions: while only leaves contribute to the predictions of a tree in standard RF, the full tree structure contributes to predictions in WW.
An illustration of this effect is given in Figure~\ref{fig:aggregation_effect} on a toy binary classification example, where we can observe that subtrees aggregation leads to improved and regularized decision functions for each individual tree and for the forest.
\begin{figure}[htbp]
    \centering
    \includegraphics[width=0.9\textwidth]{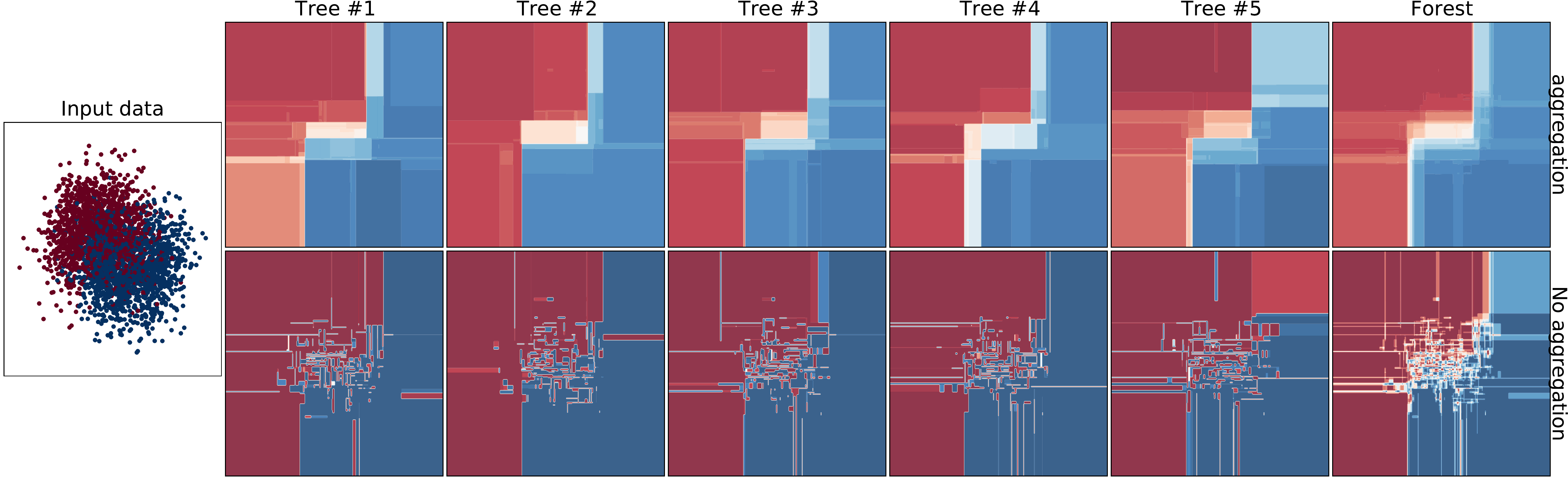}
    \caption{WW decision functions illustrated on a toy dataset (left) with subtrees aggregation (top) and without it (bottom). Subtrees aggregation improves trees predictions, as illustrated by smoother decision functions in the top compared with the bottom, improving overall predictions of the forest (last column).}
    \label{fig:aggregation_effect}
\end{figure}
We further illustrate in Figure~\ref{fig:n_trees} that each tree becomes a stronger learner, and that excellent performance can be achieved even when WW uses few trees, granting better interpretability of the solutions as a side benefit. 
Indeed, reducing the number of trees is a known way to obtain explainable models~\citep{zhou2018approximation,johansson2011one}.
A remarkable aspect of WW is that this improvement comes only at a small computational cost, thanks to a technique called ``context tree weighting'', used in lossless compression or online learning to aggregate all subtrees of a given tree~\citep{willems1995context-basic,willems1998context-extensions,helmbold1997pruning,catoni2004statistical,mourtada2019amf}.
Also, the predictions of WW do not rely on MCMC approximations required with Bayesian variants of RF~\citep{chipman1998bayesiancart,denison1998bayesiancart,chipman2010bart,taddy2011dynamictrees},  which is a clear distinction from such methods.
\begin{figure}[htbp]
    \centering
    \includegraphics[width=\textwidth]{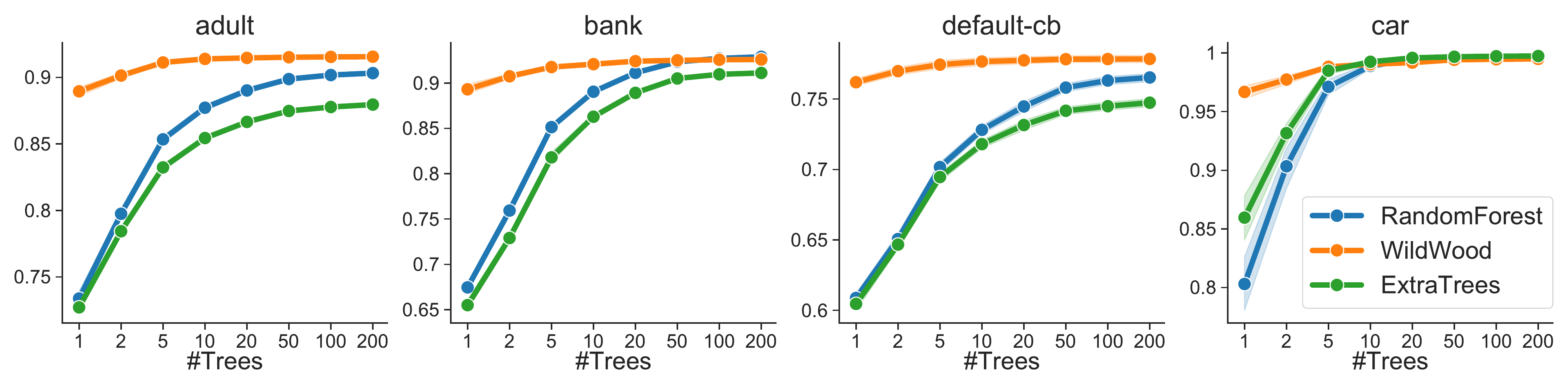}
    \caption{Mean test AUC and standard-deviations ($y$-axis) using 10 train/test splits for WW and \texttt{scikit-learn}'s implementations of RF~\citep{louppe2014understanding} and Extra Trees~\citep{geurts2006extremely}, using default hyperparameters, on several datasets. Thanks to subtrees aggregation, WW improves these baselines, even with few trees ($x$-axis is the number of trees).}
    \label{fig:n_trees}
\end{figure}
Secondly, WW uses feature binning (``histogram'' strategy), similarly to extreme gradient boosting (EGB) libraries such as XGBoost~\citep{xgboost_paper}, LightGBM~\citep{lightgbm_paper} and CatBoost~\citep{prokhorenkova2017catboost, dorogush2018catboost}.
This strategy helps accelerate computations in WW compared with standard RF algorithms, that typically require to sort features locally in nodes and try a larger number of splits~\citep{louppe2014understanding}.
This combination of subtrees aggregation and of the histogram strategy makes WW comparable with state-of-the-art implementations of EGB libraries, as illustrated in Figure~\ref{fig:auc_timings}.
\begin{figure}[htbp]
    \centering
    \includegraphics[width=\textwidth]{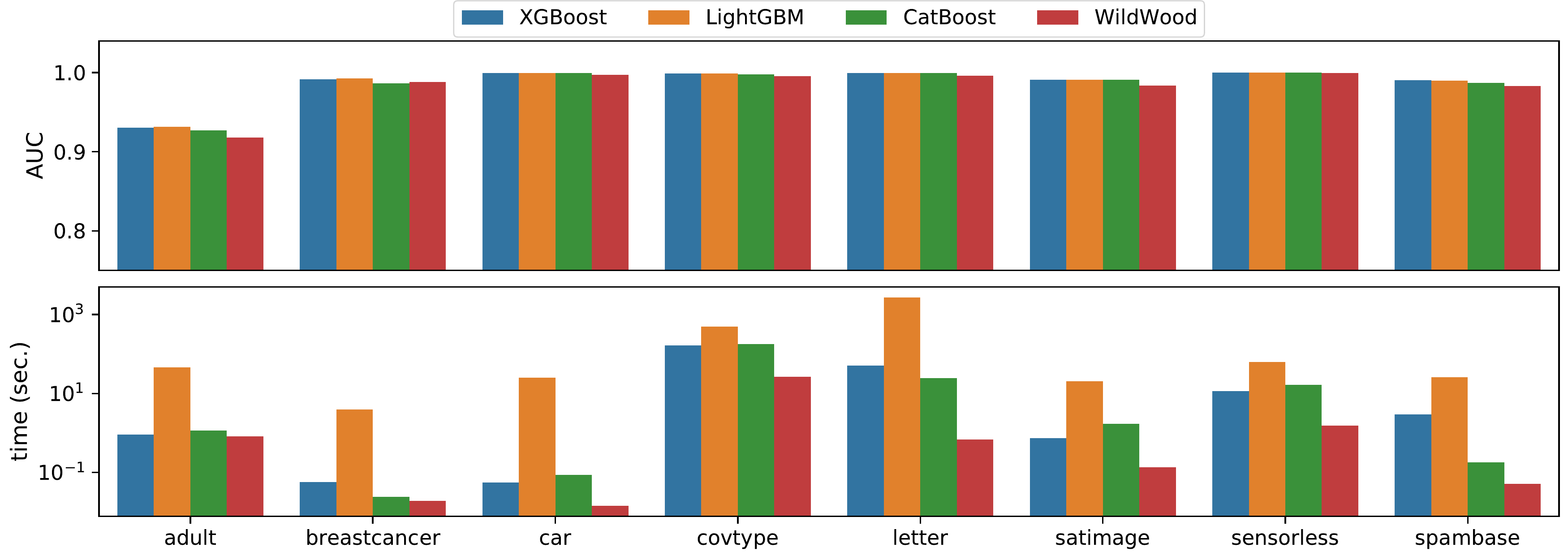}
    \caption{Test AUC (top) and training time (bottom) of WW compared with very popular EGB libraries (after hyperoptimization of all algorithms, see Section~\ref{sec:experiments} for details). 
    WW's performance, which uses only 10 trees in this display, is only slightly below such strong baselines, but is  faster (training times are on a logarithmic scale) on the considered datasets.}
    \label{fig:auc_timings}
\end{figure}
Moreover, WW supports optimal split finding for categorical features and missing values, with no need for particular pre-processing (such as one-hot encoding, see~\cite{xgboost_paper} or target encoding, see~\cite{prokhorenkova2017catboost,dorogush2018catboost}).
Finally, WW is supported by some theoretical evidence, since we prove that for a general loss function, the subtrees aggregation considered in WW leads indeed to a performance close to that of the best subtree.

\paragraph{Related works.}
Tree based decision models appeared as a supervised learning tool in the 1960s~\citep{morgan1963problems, messenger1972modal, quinlan1986induction}. Their use was promoted by the in-depth practical and theoretical study of~\cite{breiman1984cart}. The key idea of forming tree ensembles led to the well-known Random Forest algorithm~\citep{ho1995random, breiman2001randomforests} which became one of the most popular supervised learning algorithms thanks to its ease of use, robustness to hyperparameters~\citep{biau2016rf_tour,probst2019hyperparameters} and applicability to a wide range of domains. Recent examples include 
predictive medicine~\citep{subasi2017diagnosis,alghatani2021predicting}, intrusion detection~\citep{chen2021adasynrandom}, 
car safety~\citep{tavakoli2021driver}, differential privacy~\citep{6968348} and COVID-19~\citep{2020arXiv201101219S} to cite but a few. 

The RF algorithm sparked considerable interest and benefited from a rich literature of methodological developments~\citep{geurts2006extremely, criminisi2012decision, athey2018generalized}. In particular, RF was adapted for various tasks including the computation of prediction intervals~\citep{doi:10.1177/0962280219829885, doi:10.1080/00031305.2019.1585288, 10.1007/978-3-030-62362-3_16}, online learning~\citep{lakshminarayanan2014mondrianforests,mourtada2019amf} and quantile regression~\citep{meinshausen2006quantile}. A theoretical framework was progressively built around random forests to study their statistical properties such as consistency~\citep{biau2008consistency, biau2012analysis, scornet2015consistency_rf}, control over bias and variance~\citep{arlot2014purf_bias, genuer2012variance} and feature importance assessment~\citep{louppe2013understanding, kazemitabar2017variable}. These works often consider purely random trees (splits are not optimized using training data) which are more amenable to theoretical analysis. More recent works introduced ``honest'' trees which are built using disjoint samples from those used for prediction~\citep{athey2016recursive}. Honest forests were considered and proven consistent for causal inference tasks by~\cite{wager2018estimation}.

By essence, decision trees are universal approximators which lends them strong expressive power but also makes them prone to overfitting, hence the need for regularization. This was addressed in early works using pruning methods such as CCP, REP or MEP~\citep{quinlan1986induction, breslow1997simplifying}. Although they are fairly effective at reducing tree overfitting, these methods are mostly based on heuristics so that little is known about their theoretical properties. 
A form of soft-pruning was proposed by~\cite{buntine1992learning} and referred to as \emph{tree smoothing}. 
The latter efficiently computes predictions as approximate Bayesian posteriors over the set of possible prunings, however, the associated complexity is of the order of the tree-size, which makes the computation of predictions slow. An alternative form of regularization called ``shrinkage'' was proposed in~\citep{hastie1990shrinking, agarwal2022hierarchical} and consists in averaging predictions between leaf nodes and their ancestors. Although a parallel can be drawn between such methods and kernel Ridge regression, their theoretical support remains limited and does not allow to claim improved performance. Recent studies suggest that tree depth limitation and randomness of the tree building process bring an implicit regularization effect which improves performance on low signal-to-noise ratio datasets~\citep{mentch2020randomization, zhou2021trees}.

In~\citep{mourtada2019amf}, an improvement of Mondrian Forests~\citep{lakshminarayanan2014mondrianforests} is introduced for online learning, using subtrees aggregation with exponential weights, which is particularly convenient in the online learning setting.
However,~\cite{mourtada2019amf} considers only the online setting, with purely random trees, leading to poor performances compared with realistic decision trees.

Extreme boosting algorithms are another family of tree ensemble methods.
XGBoost~\citep{xgboost_paper} provides an extremely popular scalable tree boosting system which has been widely adopted in industry. 
LightGBM~\citep{lightgbm_paper} employed the ``histogram strategy'' for faster split finding, together with clever downsampling and features grouping algorithms in order to achieve high performance in reduced computation times.
CatBoost~\citep{probst2019hyperparameters} is another boosting library which pays particular attention to categorical features using target encoding, while addressing the potential bias issues associated to such an encoding.

In WW, we combine several ideas from the literature into a novel high performing RF algorithm.
We use a subtrees aggregation mechanism similar to~\citep{mourtada2019amf} for batch learning in a different way: we exploit the bootstrap, one of the key ingredients of RF, which provides in-the-bag and out-of-bag samples, to perform aggregation with exponential weights leading to considerably improved predictions. In addition, we employ the histogram strategy~\citep{lightgbm_paper} in order to grow decision trees much more efficiently. Finally, we implement an improved split finding method along categorical features which makes for better predictions.

\paragraph{Limitations.}
Our implementation of WW is still evolving and is not yet at the level of maturity of state-of-the-art EGB libraries such as~\cite{xgboost_paper,lightgbm_paper,probst2019hyperparameters}.
It does not outperform such strong baselines, but proposes an improvement of RF algorithms, and gives an interesting balance between performance and computational efficiency.

\section{WildWood: a new Random Forest algorithm}
\label{sec:wildwood}

We consider batch supervised learning, where data comes as a set of i.i.d training samples $(x_i, y_i)$ for $i=1, \ldots, n$ with vectors of numerical or categorical features $x_i \in \cX \subset \R^d$ and $y_i \in \cY$.
Our aim is to design a RF predictor $\widehat g(\: \cdot \: ; \boldsymbol \Pi) = \frac 1M \sum_{m=1}^M \widehat f(\: \cdot \: ; \Pi_m) : \cX \to \widehat \cY$
computed from training samples, where $\widehat \cY$ is the prediction space.
Such a RF computes the average of $M$ randomized trees predictions $\widehat f(\: \cdot \: ; \Pi_m)$ following the principle of bagging~\citep{breiman1996bagging, oza2001onlinebagging}, with $\boldsymbol \Pi = (\Pi_1, \ldots, \Pi_M)$ where $\Pi_1, \ldots, \Pi_M$ are i.i.d realizations of a random variable corresponding to bootstrap and feature subsampling (see Section~\ref{sub:tree} below).
Each tree is trained independently of each other, in parallel.
In what follows we describe only the construction of a single tree and omit from now on the dependence on $m=1, \ldots, M$.

\paragraph{Feature binning.}

The split finding strategy described in Section~\ref{sub:splits} below works on binned features. 
While this technique is of common practice in EGB libraries~\citep{xgboost_paper, lightgbm_paper,  prokhorenkova2017catboost}, we are not aware of an implementation of it for RF. Note that binning may affect predictive performance by making split thresholds less accurate. However, the experimental performance from EGB libraries (see also Section~\ref{sec:experiments}) suggests that this effect is minor.
The input $n \times d$ matrix $\bs X$ of features is transformed into another same-size matrix of ``binned'' features denoted $\bs X^{\text{bin}}$.
To each input feature $j=1, \ldots, d$ is associated a set $B_j = \{ 1, \ldots, b_j \}$ of bins, where $b_j \leq b_{\max}$ with $b_{\max}$ a hyperparameter corresponding to the maximum number of bins a feature can use (default is $b_{\max} = 256$ similarly to~\cite{lightgbm_paper}, so that a single byte can be used for entries of $\bs X^{\text{bin}}$).
When a feature is continuous, it is binned into $b_{\max}$ bins using inter-quantile intervals.
If it is categorical, each modality is mapped to a bin whenever $b_{\max}$ is larger than its number of modalities, otherwise sparsest modalities end up binned together.
If a feature $j$ contains missing values, its rightmost bin in $B_j$ is used to encode them 
(in such case, later in split, we do not loop only left to right (along bin order), but right to left as well, in order to compare splits that put missing values on the left or on the right).
After binning, each column satisfies $\bs X_{\bullet, j}^\text{bin} \in B_j^n$.

\subsection{Random decision trees}
\label{sub:tree}

Let $\cell = \prod_{j=1}^d B_j$ be the binned feature space.
A random decision tree is a pair $(\tree, \splits)$, where $\tree$ is a finite ordered binary tree and $\splits$ contains information about each node in $\tree$, such as split information.
The tree is random and its source of randomness $\Pi$ comes from the bootstrap and feature subsampling as explained below.

\paragraph{Finite ordered binary trees.} 

A finite ordered binary tree $\tree$ is represented as a finite subset of the set $\{ 0, 1 \}^* = \bigcup_{n \geq 0} \{ 0, 1 \}^n$ of all finite words on $\{ 0, 1\}$.
The set $\{ 0, 1\}^*$ is endowed with a tree structure (and called the complete binary tree): the empty word $\root$ is the root, and for any $\node \in \{ 0, 1\}^*$, the left (resp. right) child of $\node$ is $\node 0$ (resp. $\node 1$).
We denote by $\inodes (\tree) = \{ \node \in \tree : \node 0, \node 1 \in \tree \}$ the set of its interior nodes and by $\leaves (\tree) = \{ \node \in \tree : \node 0, \node 1 \not\in \tree \}$ the set of its leaves, both sets are disjoint and the set of all nodes is $\nodes(\tree) = \inodes (\tree) \cup \leaves (\tree)$.

\paragraph{Splits and cells.} 

The split $\sigma_\node = (j_\node, t_\node) \in \Sigma$ of each $\node \in \inodes (\tree)$ is characterized by its dimension $j_\node \in \{ 1, \dots, d \}$ and a non-empty subset of bins $t_\node \subsetneq \{ 1, \ldots, b_{j_\node} \}$.
We associate to each $\node \in \tree$ a cell $\cell_\node \subseteq C$ which is defined recursively: $C_\root = C$ and for each $\node \in \inodes(\tree)$ we define
\begin{equation*}
  \cell_{\node 0} := \{ x \in \cell_\node : x_{j_\node} \in t_{\node}  \} \quad \text{and} \quad \cell_{\node 1} := \cell_\node \setminus \cell_{\node 0}.
\end{equation*}
When $j_\node$ corresponds to a continuous feature, bins have a natural order so that $t_\node$ is encoded by a bin threshold  $s_{\node} \in B_{j_\node}$ as $t_\node = \{ 1, 2, \ldots, s_{\node} \}$; while for a categorical split, the set $t_\node$ may be any non-trivial subset of $B_{j_\node}$.
By construction, $(\cell_{\leaf})_{\leaf \in \leaves (\tree)}$ is a partition of $\cell$. 

\paragraph{Bootstrap and feature subsampling.}

Let $I = \{1, \ldots, n\}$ be the training samples indices.
The randomization $\Pi$ of the tree uses bootstrap: it samples uniformly at random, with replacement, elements of $I$ corresponding to in-the-bag (\itb) samples.
If we denote as $I_\itb$ the indices of unique \itb samples, we can define the indices of out-of-bag (\otb) samples as $I_\otb = I \setminus I_\itb$.
A standard argument shows that $\P[i \in I_\itb] = 1 - (1 - 1/n)^n \rightarrow 1 - e^{-1} \approx 0.632$ as $n \rightarrow +\infty$, known as the 0.632 rule~\citep{efron1997improvements}.
The randomization $\Pi$ uses also feature subsampling: each time we need to find a split, we do not try all the features $\{1, \ldots, d\}$ but only a subset of them of size $d_{\max}$, chosen uniformly at random.
This follows what standard RF algorithms do~\citep{breiman2001randomforests,biau2016rf_tour,louppe2014understanding}, with the default $d_{\max} = \sqrt{d}$.

\subsection{Split finding on histograms}
\label{sub:splits}

For $K$-class classification, when looking for a split for some node $\node$, we compute the node's ``histogram'' 
$\mathrm{hist}_\node[j, b, k] = \sum_{i \in I_\itb : x_i \in \cell_\node} \ind{x_{i, j} = b, y_i = k}$ for each sampled feature $j$, each bin $b$ and label class $k$ seen in the node's samples (actually weighted counts to handle bootstrapping and sample weights).
Of course, one has $\mathrm{hist}_\node = \mathrm{hist}_{\node 0} + \mathrm{hist}_{\node 1}$, so that we don't need to compute two histograms for siblings $\node 0$ and $\node 1$, but only a single one. 
Then, we loop over the set of non-constant (in the node) sampled features $\{ j : \# \{ b : \sum_{k} \mathrm{hist}_\node[j, b, k] \geq 1 \} \geq 2 \}$ and over the set of non-empty bins $\{ b : \sum_{k} \mathrm{hist}_\node[j, b, k] \geq 1 \}$ to find a split, by comparing standard impurity criteria computed on the histogram's statistics, such as gini or entropy for classification and variance for regression.

\paragraph{Bin order and categorical features.}

The order of the bins used in the loop depends on the type of the feature.
If it is continuous, we use the natural order of bins.
If it is categorical and the task is binary classification (labels in $\{0, 1\}$) we use the bin order that sorts $\mathrm{hist}_\node[j, b, 1] / \sum_{k=0, 1} \mathrm{hist}_\node[j, b, k]$ with respect to $b$, namely the proportion of labels $1$ in each bin. 
This allows to find the optimal split with complexity $O(b_j \log b_j)$, see Theorem~9.6 in~\citep{breiman1984cart}, the logarithm coming from the sorting operation, while there are $2^{b_j - 1} -1$ possible splits.
This trick is used by EGB libraries as well, using an order of $\text{gradient} / \text{hessian}$ statistics of the loss considered~\citep{xgboost_paper, lightgbm_paper, prokhorenkova2017catboost}. 
For $K$-class classification with $K > 2$, we consider two strategies: 
(1)~one-versus-rest, where we train $M K$ trees instead of $M$, each tree trained with a binary one-versus-rest label, so that trees can find optimal categorical splits and 
(2)~heuristic, where we train $M$ trees and where split finding uses $K$ loops over bin orders that sort $\mathrm{hist}_\node[j, b, k] / \sum_{k'} \mathrm{hist}_\node[j, b, k']$ (w.r.t $b$) for $k=0, \ldots, K-1$. 
The former strategy generally yields better predictions but at a higher computational cost, while the latter is more efficient but also more prone to overfitting.

\paragraph{Split requirements.}

Nodes must hold at least one $\itb$ and one $\otb$ sample to apply aggregation with exponential weights, see Section~\ref{sub:agg-ctw} below.
A split is discarded if it leads to children with less than $n_{\text{min-leaf}}$ \itb or \otb samples and we do not split a node with less than $n_{\text{min-split}}$ \itb or \otb samples.
These hyperparameters only weakly impact WW's performances and sticking to default values ($n_{\text{min-leaf}}=1$ and $n_{\text{min-split}}=2$, following \texttt{scikit-learn}'s~\citep{louppe2014understanding,pedregosa2011scikit-learn}) is usually enough (see Section~\ref{apd:hyperparameters-sensitivity} below).

\paragraph{Related works on categorical splits.}

In~\cite{PartitioningNominal}, an interesting characterization of an optimal categorical split for multiclass classification is introduced, but no efficient algorithm is, to the best of our understanding, available for it. 
A heuristic algorithm is proposed therein, but it requires to compute, for each split, the top principal component of the covariance matrix of the conditional distribution of labels given bins, which is computationally too demanding for an RF algorithm intended for large datasets.
Regularized target encoding is shown in~\cite{pargent2021regularized} to perform best when compared with many alternative categorical encoding methods.
Catboost~\citep{prokhorenkova2017catboost}
uses target encoding, which replaces feature modalities by label statistics, so that a natural bin order can be used for split finding.
To avoid overfitting on uninformative categorical features, a debiasing technique uses random permutations of samples and computes the target statistic of each element based only on its predecessors in the permutation. 
However, for multiclass classification, target encoding is influenced by the arbitrarily chosen ordinal encoding of the labels.
LightGBM~\citep{lightgbm_paper} uses a one-versus-rest strategy, which is also one of the approaches used in WW for categorical splits on multiclass tasks.
For categorical splits, where bin order depends on labels statistics, WW does not use debiasing as in~\cite{prokhorenkova2017catboost}, since aggregation with exponential weights computed on \otb samples allows to deal with overfitting.

\paragraph{Tree growth stopping.}

We do not split a node and make it a leaf if it contains less than $n_{\text{min-split}}$ \itb or \otb samples. 
The same applies when a node's impurity is not larger than a threshold $\eps$ ($\eps=0$ by default).
When only leaves or non-splittable nodes remain, the growth of the tree is stopped.
Trees grow in a depth-first fashion so that childs $\node 0$ and $\node 1$ have memory indexes larger than their parent $\node$ (as required by Algorithm~\ref{alg:wbar-computation} below). In practice, trees are grown to full depth in WildWood (i.e. we take $n_{\text{min-split}}=2$) since we can only benefit from additional splits and do not need to worry about overfitting since it is prevented by the subtrees aggregation we present now.

\subsection{Prediction function: aggregation with exponential weights}
\label{sub:agg-ctw}

Given a tree $\tree$ grown as described in Sections~\ref{sub:tree} and~\ref{sub:splits}, its prediction function is an aggregation of the predictions given by all possible subtrees rooted at $\root$, denoted $\{T : T \subset \tree \}$. Figure~\ref{fig:subtree_illustr} provides an example of a tree and two such subtrees. 
While $\tree$ is grown using \itb samples, we use \otb samples to perform aggregation with exponential weights, with a branching process prior over subtrees. Thus, the aggregation mechanism gives more importance to shallow subtrees with a good predictive \otb performance. The balance between shallowness and prediction quality is modulated through a temperature parameter denoted $\eta$ in Equation~\eqref{eq:exact-aggregation} below.
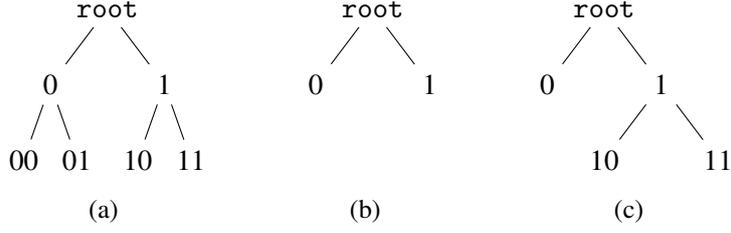
\begin{figure}
    \centering
    \subfloat[][]{\centering
    \begin{tikzpicture}[level distance=1cm]
        \node{$\mathtt{root}$}
        child {node {0} child { node[xshift=4mm] {00} } child { node[xshift=-4mm] {01}  }}
        child {node {1} child { node[xshift=4mm] {10} } child { node[xshift=-4mm] {11}  }};
    \end{tikzpicture}\label{subfig:a}
    }\hspace{1cm}%
    \subfloat[][]{
    \centering
    \begin{tikzpicture}[level distance=1cm]
        \node{$\mathtt{root}$}
        child {node {0} child[white]{node {\phantom{0}}}}
        child {node {1} };
    \end{tikzpicture}\label{subfig:b}
    }\hspace{1cm}%
    \subfloat[][]{
    \centering
        \begin{tikzpicture}[level distance=1cm]
            \node{$\mathtt{root}$}
            child {node {0} }
            child {node {1} child { node {10} } child { node {11}  }};
        \end{tikzpicture}\label{subfig:c}
    }
    \caption{Example of a simple tree $\tree$ show in~(\protect\subref{subfig:a}) and two subtrees $T_1, T_2 \subset \tree$ rooted at $\root$ shown in~(\protect\subref{subfig:b}) and~(\protect\subref{subfig:c}). }
    \label{fig:subtree_illustr}
\end{figure}
\paragraph{Node and subtree prediction.}

We define $\node_{T} (x) \in \leaves(T)$ as the leaf of $T$ containing $x \in \cell$.
The prediction of a node $\node \in \nodes(\tree)$ and of a subtree $T \subset \tree$ is given by
\begin{equation}
    \label{eq:node_subtree_prediction}
    \pred_{\node} = h ( (y_i)_{i \in I_\itb \pp x_i \in \cell_\node}) \quad \text{ and } \quad  \pred_{T} (x) = \pred_{\node_{T} (x)},
\end{equation}
where $h : \cup_{n \geq 0} \cY^n \to \widehat \cY$ is a generic ``forecaster'' used in each cell and where a subtree prediction is that of its leaf containing $x$.
A standard choice for regression ($\cY = \widehat \cY = \R$) is the empirical mean forecaster
\begin{equation}
    \label{eq:reg-predictor}
    \pred_{\node} = \frac{1}{n_{\node}} 
    \sum_{i \in I_\itb \pp x_i \in \cell_\node} y_i,
\end{equation}
where $n_{\node} = | \{i \in I_\itb \pp x_i \in \cell_\node \} |$.
For $K$-class classification with $\cY = \{ 1, \ldots, K \}$ and $\widehat \cY = \probas(\cY)$, the set of probability distributions over $\cY$, a standard choice is a Bayes predictive posterior with a prior on $\probas (\cY)$ equal to the Dirichlet distribution $\dirichletdist(\alpha, \dots, \alpha)$, namely the \emph{Jeffreys prior} on the multinomial model $\probas (\cY)$, which leads to
\begin{equation}
  \label{eq:kt-predictor}
  \pred_{\node} (k) = \frac{n_{\node} (k) + \alpha}{n_{\node} + \alpha K},
\end{equation}
for any $k \in \cY$, where $n_{\node} (k) = | \{ i \in I_\itb : x_i \in \cell_\node, y_i = k \} |$.
By default, WW uses $\alpha = 1/2$ (the \emph{Krichevsky-Trofimov} forecaster~\citep{tjalkens1993sequential}), but one can perfectly use any $\alpha > 0$, so that all the coordinates of $\pred_{\node}$ are positive.
This is motivated by the fact that WW uses as default the log loss to assess $\otb$ performance for classification, which requires an arbitrarily chosen clipping value for zero probabilities.
Different choices of $\alpha$ only weakly impact WW's performance, as illustrated in Appendix~\ref{apd:hyperparameters-sensitivity}.
We use \otb samples to define the cumulative losses of the predictions of all $T \subset \tree$
\begin{equation}
    \label{eq:subtree_loss}
    L_T =  \sum_{i \in I_\otb} \ell (\pred_{T} (x_i), y_i),
\end{equation}
where $\loss : \widehat \cY \times \cY \to \R^+$ is a loss function.
For regression problems, a default choice is the quadratic loss $\ell (\pred, y) = (\pred - y)^2$ while for multiclass classification, a default is the log-loss $\ell (\pred, y) = - \log \pred(y)$, where $\pred(y) \in (0, 1]$ when using~\eqref{eq:kt-predictor}, but  other loss choices are of course possible.

\paragraph{Prediction function.}

Let $x \in \cell$.
The prediction function $\widehat f$ of a tree $\tree$ in WW is given by
\begin{equation}
  \label{eq:exact-aggregation}
 \widehat f (x) = \frac{\sum_{T \subset \tree} \pi (T) e^{-\eta L_T} \pred_{T} (x)}{\sum_{T \subset \tree} \pi (T) e^{-\eta L_T}} \quad \text{with} \quad \pi(T) = 2^{- \| T \|},
\end{equation}
where the sum is over all subtrees $T$ of $\tree$ rooted at $\root$, where $\eta > 0$ is a temperature parameter and $\|T\|$ is the number of nodes in $T$ minus its number of leaves that are also leaves of $\tree$.
A default choice for the temperature is $\eta = 1$ for the log-loss (see in particular Corollary~\ref{cor:bound-best-pruning-log} in Section~\ref{sec:theory} below), but it can also be tuned through hyperoptimization, although we do not observe strong performance gains, see Appendix~\ref{apd:hyperparameters-sensitivity} below.

The prediction function~\eqref{eq:exact-aggregation} corresponds to an exponentially weighted average over the predictors $\wh y_T(x)$ for $T \subset \tree$ rooted at $\root$ based on \otb performance and the branching process prior $\pi$ with branching probability $1 / 2$. 
The latter being expressly chosen for its fitness within this paradigm. 
This definition falls within the framework of PAC-Bayesian theory~\citep{mcallester1998pacbayes, mcallester1999pacbayes, catoni2007pacbayes} with theoretical guarantee of the oracle inequality for $\widehat f$ on \otb (see Theorem~\ref{thm:subtree-bound} below). 
This aggregation procedure can be understood as a \emph{non-greedy way to prune trees}: the weights depend not only on the quality of one single split but also on the performance of each subsequent split.

Naively computing $\widehat f$ from Equation~\eqref{eq:exact-aggregation} is computationally and memory-wise infeasible for a large $\tree$, since it involves a sum over all $T \subset \tree$ rooted at $\root$ and requires one weight for each $T$.
Indeed, the number of subtrees of a minimal tree that separates $n$ points is exponential in the number of nodes, and hence \emph{exponential in $n$}.
However, it turns out that one can compute $\widehat f$ exactly and very efficiently thanks to the prior choice $\pi$ together with an adaptation of \emph{context tree weighting}~\citep{willems1995context-basic,willems1998context-extensions,helmbold1997pruning,catoni2004statistical}.
\begin{theorem}
    \label{thm:ctw}
    The prediction function~\eqref{eq:exact-aggregation} can be written as $\wh f(x) = \wh f_{\root}(x)$\textup, where $\wh f_{\root}(x)$ satisfies the recursion
    \begin{equation}
        \label{eq:f_pred_recursion}
        \wh f_\node(x) = \frac 12 \frac{w_{\node}}{\wbar_\node} \pred_{\node} + \Big(1 - \frac 12  \frac{w_{\node}}{\wbar_\node} \Big) \wh f_{\node a}(x)
    \end{equation}
    for $\node, \node a \in \pathpoint(x)$ \textup($a \in \{0, 1\}$\textup) the path in $\tree$ going from $\root$ to $\node_\tree(x)$\textup, where $w_\node := \exp(-\eta L_\node)$ with $L_\node := \sum_{i \in I_\otb : x_i \in C_\node} \ell (\pred_{\node}, y_i)$ and where $\wbar_\node$ are weights satisfying the recursion
    \begin{equation}
        \label{eq:wden-recursion}
        \wbar_{\node} =
        \begin{cases}
          w_{\node} & \text{ if } \node \in \leaves (\tree), \\
          \frac{1}{2} w_{\node} + \frac{1}{2} \wbar_{\node 0} \wbar_{\node 1} &\text{ otherwise}.
        \end{cases}
    \end{equation}
\end{theorem}
The proof of~Theorem~\ref{thm:ctw} is given in Section~\ref{sec:proof-thm-ctw} below, a consequence of this Theorem being a very efficient method to compute $\wh f(x)$ described in Algorithms~\ref{alg:wbar-computation} and~\ref{alg:tree-prediction} below.
Algorithm~\ref{alg:wbar-computation} computes the weights $\wbar_\node$ using the fact that trees in WW are grown in a depth-first fashion, so that we can loop \emph{once}, leading to a $O(|\nodes(\tree)|)$ complexity in time and in  memory usage, over nodes from a data structure that respects the parenthood order.
Direct computations can lead to numerical over- or under-flows (many products of exponentially small or large numbers are involved), so Algorithm~\ref{alg:wbar-computation} works recursively over the logarithms of the weights (line~6 uses a log-sum-exp function that can be made overflow-proof).
\begin{algorithm}[htbp]
    \caption{Computation of $\log(\wbar_\node)$ for all $\node \in \nodes(\tree)$.}
    \label{alg:wbar-computation}
    \begin{algorithmic}[1]
      \STATE \textbf{Inputs:} $\tree$, $\eta > 0$ and losses $L_\node$ for all $\node \in \nodes(\tree)$.
      Nodes from $\nodes(\tree)$ are stored in a data structure $\mathtt{nodes}$ that respects parenthood order: for any $\node = \mathtt{nodes}[i_{\node}] \in \inodes(\tree)$ and children $\node a = \mathtt{nodes}[i_{\node a}]$ for $a \in \{0, 1\}$, we have $i_{\node a} > i_\node$.
      \FOR{$\node \in \mathrm{reversed}(\mathtt{nodes})$}
      \IF{$\node$ is a leaf}
      \STATE Put $\log(\wbar_\node) \gets -\eta L_\node$
      \ELSE
      \STATE Put  $\log(\wbar_{\node}) \gets \log( \frac{1}{2} e^{-\eta L_{\node}} + \frac{1}{2} e^{\log(\wbar_{\node 0}) + \log(\wbar_{\node 1})})$
      \ENDIF
      \ENDFOR
      \RETURN {The set of log-weights $\{ \log(\wbar_{\node}) : \node \in \nodes(\tree) \}$}
    \end{algorithmic}
\end{algorithm}
Algorithm~\ref{alg:wbar-computation} is applied once $\tree$ is fully grown, so that WW is ready to produce predictions using Algorithm~\ref{alg:tree-prediction} below.
Note that hyperoptimization of $\eta$ or $\alpha$, if required, does not need to grow $\tree$ again, but only to update $\wbar_\node$ for all $\node \in \nodes(\tree)$ with Algorithm~\ref{alg:wbar-computation}, making hyperoptimization of these parameters particularly efficient.
\begin{algorithm}[htbp]
    \caption{Computation of $\wh f(x)$ for any $x \in C$.}
    \label{alg:tree-prediction}
    \begin{algorithmic}[1]
      \STATE \textbf{Inputs:} Tree $\tree$, losses $L_\node$ and log-weights $\log(\wbar_\node)$ computed by Algorithm~\ref{alg:wbar-computation}
      \STATE Find $\node_\tree(x) \in \leaves(\tree)$ (the leaf containing $x$) and put $\node \gets \node_\tree(x)$ 
      \STATE Put $\wh f(x) \gets \pred_\node$ (the node $\node$ forecaster, such as~\eqref{eq:reg-predictor} for regression or~\eqref{eq:kt-predictor} for classification)
      \WHILE{$\node \neq \root$}
      \STATE Put $\node \gets \mathrm{parent}(\node)$
      \STATE Put $\alpha \gets \frac 12 \exp(-\eta L_\node - \log(\wbar_\node))$ 
      \STATE Put $\wh f(x) \gets \alpha \pred_\node + (1 - \alpha) \wh f(x)$
      \ENDWHILE
      \RETURN {The prediction $\wh f(x)$}
    \end{algorithmic}
\end{algorithm}
The recursion used in Algorithm~\ref{alg:tree-prediction} has a complexity $O(|\pathpoint(x)|)$ which is the complexity required to find the leaf $\node_\tree(x)$ containing $x \in C$:  Algorithm~\ref{alg:tree-prediction} \emph{only increases by a factor $2$} the prediction complexity of a standard RF (in order to go down to $\node_\tree(x)$ and up again to $\root$ along $\pathpoint(x)$).
More details about the construction of Algorithms~\ref{alg:wbar-computation} and~\ref{alg:tree-prediction} can be found in Section~\ref{sec:proof-thm-ctw} below.

\section{Theoretical guarantees}
\label{sec:theory}

This section proposes some theoretical guarantees on the subtrees aggregation~\eqref{eq:exact-aggregation} used in WW.
We say that a loss function $\ell$ is $\eta$-exp-concave for some $\eta > 0$ whenever $z \mapsto \exp(-\eta \ell(z, y))$ is concave for any $y \in \cY$.
We consider a fully-grown tree $\tree$ computed using $\itb$ samples and the set of $\otb$ samples $(x_i, y_i)_{i \in I_\otb}$ on which $L_T$ is computed using~\eqref{eq:subtree_loss}, and we denote $n_\otb := |I_\otb|$. 
\begin{theorem}[Oracle inequality]
\label{thm:subtree-bound}
    Assume that the loss function $\loss$ is $\eta$-exp-concave.
    Then\textup, the prediction function $\wh f$ given by~\eqref{eq:exact-aggregation} satisfies the oracle inequality
    \begin{equation*}
        \frac{1}{n_\otb} \sum_{i \in I_\otb} \loss(\wh f(x_i), y_i) 
        \leq \inf_{T \subset \tree} \bigg \{ \frac{1}{n_\otb} \sum_{i \in I_\otb} \loss (\pred_{T} (x_i), y_i) + \frac{\log 2}{\eta} \frac{\|T\|}{n_\otb + 1} \bigg\},
    \end{equation*}
    where the infimum is over any subtree $T \subset \tree$ and where we recall that $\|T\|$ is the number of nodes in $T$ minus its number of leaves that are also leaves of $\tree$.
\end{theorem}
Theorem~\ref{thm:subtree-bound} proves that, for a general loss function, the prediction function of WW is able to perform nearly as well as the best \emph{oracle} subtree $T \subset \tree$ on \otb samples, with a $O(\|T\| / n_{\otb})$ rate which is optimal for model-selection oracle inequalities~\citep{tsybakov2003optimal} 
($\|T\| = O(\log N_\tree)$ with a number of ``experts'' $N_\tree = |\{ T : T \subset \tree \}|$ for a well-balanced $\tree$). 
Let us stress again that, while finding an oracle $\argmin_{T \subset \tree} \sum_{i \in I_\otb} \loss (\pred_{T} (x_i), y_i)$ is computationally infeasible, since it requires to try out all possible subtrees, WW's prediction function~\eqref{eq:exact-aggregation} comes at a cost comparable to that of a standard Random Forest, as explained in Section~\ref{sub:agg-ctw} above.

The proof of Theorem~\ref{thm:subtree-bound} is given in Section~\ref{sec:proofs-oracle} below and relies on techniques from PAC-Baysesian theory~\citep{mcallester1998pacbayes, mcallester1999pacbayes, catoni2007pacbayes}. 
The arguments hinge upon the mixability of the loss, a consequence of its exp-concavity, as well as the Donsker-Varadhan variational formula.
%
Note that the statement concerns only the performance of a single tree and not the whole forest. Indeed, extrapolating the benefit brought by the exponential weights' aggregation within each tree to the forest itself remains a difficult open problem even though it is intuitive enough that improved individual trees should enhance the overall performance.
Compared with~\cite{mourtada2019amf} about online learning, our proof differs in that we do not use results specialized to online learning such as~\cite{vovk1998mixability} nor online-to-batch conversion~\citep{cesabianchi2004online_to_batch}.
Note that Theorem~\ref{thm:subtree-bound} does not address the generalization error, since it would require to study the generalization error of the random forest itself (and of the fully grown tree $\tree$), which is a topic way beyond the scope of this paper, and still a very difficult open problem: recent results~\citep{genuer2012variance, arlot2014purf_bias, scornet2015consistency_rf, scornet2016asymptotics_rf, scornet2016kernel_rf, mourtada2020minimax} only study stylized versions of RF (called purely random forests).

Consequences of Theorem~\ref{thm:subtree-bound} are Corollary~\ref{cor:bound-best-pruning-log} for the log-loss (classification) and Corollary~\ref{cor:bound-best-pruning-square} for the least-squares loss (regression).

\begin{corollary}[Classification]
\label{cor:bound-best-pruning-log}
Consider $K$-class classification \textup($\cY = \{ 1, \ldots, K \}$\textup) and consider the prediction function $\wh f$ given by~\eqref{eq:exact-aggregation}\textup, where node predictions are given by~\eqref{eq:kt-predictor} with $\alpha = 1/2$ \textup(WW's default\textup)\textup, where $\ell$ is the log-loss and where $\eta = 1$.
Then\textup, we have
\begin{equation*}
\frac{1}{n_\otb} \sum_{i \in I_\otb} \loss(\widehat f (x_i), y_i) \leq \inf_{T \subset \tree} \Big\{ \frac{1}{n_\otb} \sum_{i \in I_\otb} \loss(g_T(x_i), y_i) + \frac{K + 4 \log 2 - 1}{4} \frac{\|T\| + 1}{n_\otb} \Big\},
\end{equation*}
where $g_T$ is any constant function on the leaves of $T$.
\end{corollary}

\begin{corollary}[Regression]
\label{cor:bound-best-pruning-square}
Consider regression with $\cY = [-B,B]$ for some $B > 0$ and the prediction function $\wh f$ given by~\eqref{eq:exact-aggregation}\textup, where node predictions are given by~\eqref{eq:reg-predictor}, where $\ell$ is the least-squares loss and where $\eta = 1/(8B^2)$.
Then\textup, we have
\begin{equation*}
\frac{1}{n_\otb} \sum_{i \in I_\otb} \loss(\widehat f (x_i), y_i) \leq \inf_{T \subset \tree} \Big\{ \frac{1}{n_\otb} \sum_{i \in I_\otb} \loss(g_T(x_i), y_i) 
+ 8 (\log 2) B^2 \frac{ \|T\| }{n_\otb} \Big\},
\end{equation*}
where $g_T$ is any function constant on the leaves of $T$.
\end{corollary}
The proofs of Corollaries~\ref{cor:bound-best-pruning-log} and \ref{cor:bound-best-pruning-square} are given in Section~\ref{sec:proofs-oracle}.
These corollaries motivate the default hyperparameter values of $\eta$, in particular $\eta = 1$ for classification.

\section{Experiments}
\label{sec:experiments}

Our implementation of WildWood is available at the \texttt{GitHub} repository \url{https://github.com/pyensemble/wildwood.git} under the \texttt{BSD3-Clause} license on \texttt{GitHub} and available through \texttt{PyPi}.
It is a \texttt{Python} package that follows \texttt{scikit-learn}'s API conventions, that is JIT-compiled to machine code using \texttt{numba}~\citep{numba}. 
Trees in the forest are grown in parallel using \texttt{joblib}~\citep{joblib} and CPU threads, GPU training will be supported in future updates.

\paragraph{Baselines.} 

We compare WildWood (denoted \texttt{WW}$n$ for $n$ trees) with several strong baselines including :
\begin{itemize}
    \item \texttt{RF}$n$: \texttt{scikit-learn}'s implementation of Random Forest~\citep{pedregosa2011scikit-learn,louppe2014understanding} using $n$ trees;
    \item an elementary random forest algorithm we refer to as \texttt{RD}$n$ using $n$ trees, each grown with a randomized depth limit (uniform from $3$ to $50$), this is implemented using WildWood by deactivating aggregation and categorical feature support;
    \item \texttt{HGB}: a histogram-based implementation of extreme gradient boosting (inspired by LightGBM) from \texttt{scikit-learn};
    \item several state-of-the-art and widely adopted extreme gradient boosting libraries including \texttt{XGB}:  XGBoost~\citep{xgboost_paper}; \texttt{LGBM}: LightGBM~\citep{lightgbm_paper} and \texttt{CB}: CatBoost~\citep{prokhorenkova2017catboost,dorogush2018catboost}.
\end{itemize}  
Note that we focus on computationally competitive algorithms which excludes MCMC based methods such as BART~\citep{chipman1998bayesiancart}. 

\paragraph{Computational resources.} 

We used a 32-cores server with two Intel Xeon Gold CPUs, two Tesla V100 GPUs and 384GB RAM
for the experiments involving hyperoptimization (Table~\ref{tab:performance}) and used a 12-cores Intel i7 MacBook Pro with 32GB RAM and no GPU to obtain training times achievable by a ``standard user'' (Table~\ref{tab:fit_time}).
All experiments can be reproduced using \texttt{Python} scripts on the repository.

\paragraph{Data.}

We use publicly available and open-source datasets from the UCI repository~\citep{Dua:2019}, including small datasets (hundreds of rows) and large datasets (millions of rows), their main characteristics are given in Table~\ref{tab:dataset-descrip} together with URLs in Table~\ref{tab:data-source}, see Section~\ref{sec:datasets} below.
Each dataset is randomly split into a training set (70\%) and a test set (30\%).
We specify which features are categorical to algorithms that natively support it (\texttt{HGB}, \texttt{LGBM}, \texttt{CB} and \texttt{WW}$n$) and simply integer-encode them, while we use one-hot encoding for other algorithms (\texttt{RF}$n$, \texttt{RD}$n$, \texttt{XGB}).
\subsection{Performance on classification tasks}
For each dataset and algorithm, we evaluate the performance after tuning the hyperparameters. Hyperoptimization is performed as follows: from the training set, we use $4/5$ for training and $1/5$ for validation and do 50 steps of sequential optimization using the Tree Parzen Estimator implemented in the \texttt{hyperopt} library~\citep{bergstra2015hyperopt}.
More details about hyperoptimization are provided in Section~\ref{sec:exp-details} below.
Then, we refit on the whole training set with the best hyperparameters and report scores on the test set.
This is performed 5 times in order to report standard deviations. We use the area under the ROC curve (AUC), for $K$-class datasets with $K > 2$ we average the AUC of each class versus the rest.
This leads to the test AUC scores displayed in Table~\ref{tab:performance} (the same scores with standard deviations are available in Table~\ref{tab:performance-with-std}). 
\begin{table}[htbp]
\caption{Test AUC of all algorithms after hyperoptimization on the considered datasets. Standard-deviations are reported in Table~\ref{tab:performance-with-std} and results in terms of log-loss in Table~\ref{tab:performance-logloss}. 
We observe that WW has better (or identical in some cases) performances than RF on all datasets and that it is close to that of EGB libraries (bold is for best EGB performance, underline for best \texttt{RF}$n$, \texttt{WW}$n$ or \texttt{RD}$n$ performance).}
\resizebox{\textwidth}{!}{\begin{tabular}{l|cccc|cccccc}
\toprule
 &   \texttt{XGB} &  \texttt{LGBM} & \texttt{CB} & \texttt{HGB} & \texttt{RF}$10$ & \texttt{RF}$100$ & \texttt{WW}$10$ & \texttt{WW}$100$ & \texttt{RD}$10$ & \texttt{RD}$100$ \\
\midrule
adult        &  0.930 &  \bestb{0.931} &  0.927 &  0.930 &  0.915 & 0.918 & 0.916 & \bestf{0.919} & 0.915 & 0.917 \\
bank         &  0.933 &  \bestb{0.935} &  0.925 &  0.930 &  0.919 &  0.929 &  0.926 &  \bestf{0.931} &  0.919 &  0.922 \\
breastcancer &  0.991 &  0.993 &  0.987 &  \bestb{0.994} &  0.987 &  0.989 & \bestf{0.992} &  \bestf{0.992} &  0.987 &  0.985 \\
car          &  0.999 &  \bestb{1.000} &  \bestb{1.000} &  \bestb{1.000} &  0.997 &  \bestf{0.998} & \bestf{0.998} &  \bestf{0.998} &  0.993 &  0.993 \\
covtype      &  \bestb{0.999} &  \bestb{0.999} & 0.998 &  \bestb{0.999} &  0.997 &  \bestf{0.998} &  0.996 &  \bestf{0.998} &  0.974 &  0.996 \\
default-cb   &  0.780 &  \bestb{0.783} &  0.780 &  0.779 &  0.765 &  0.775 &  0.774 &  \bestf{0.778} &  0.772 &  0.773 \\
higgs        &  0.853 &  \bestb{0.857} &  0.847 &  0.853 &  0.820 &  \bestf{0.837} &  0.820 &  \bestf{0.837} &  0.813 &  0.815 \\
internet     &  0.934 &  0.910 &  \bestb{0.938} &  0.911 &  0.917 &  \bestf{0.935} &  0.926 &  0.928 &  0.925 &  0.928 \\
kddcup       &  \bestb{1.000} &  \bestb{1.000} &  \bestb{1.000} &  \bestb{1.000} &  0.997 &  0.998 &  \bestf{1.000} & \bestf{1.000}  &  \bestf{1.000} &  \bestf{1.000} \\
kick         &  \bestb{0.777} &  0.770 &  \bestb{0.777} &  0.771 &  0.749 &  \bestf{0.764} &  0.756 &  0.763 &  0.752 &  0.754 \\
letter       &  \bestb{1.000} &  \bestb{1.000} &  \bestb{1.000} &  \bestb{1.000} &  0.997 &  \bestf{0.999} &  0.997 &  \bestf{0.999} &  0.995 &  0.993 \\
satimage     &  \bestb{0.991} &  \bestb{0.991} &  \bestb{0.991} &  0.987 &  0.985 &  \bestf{0.991} &  0.986 &  \bestf{0.991} &  0.984 &  0.983 \\
sensorless   &  \bestb{1.000} &  \bestb{1.000} &  \bestb{1.000} &  \bestb{1.000} &  \bestf{1.000} &  \bestf{1.000} &  \bestf{1.000} &  \bestf{1.000} &  0.999 &  \bestf{1.000} \\
spambase     &  \bestb{0.990} &  \bestb{0.990} &  0.987 &  0.986 &  0.980 &  \bestf{0.987} &  0.983 &  \bestf{0.987} &  0.972 & 0.974 \\
\bottomrule
\end{tabular}}
\label{tab:performance}
\end{table}

We observe in Table~\ref{tab:performance} that EGB algorithms, when hyperoptimized, lead to the best performances over the considered datasets compared with RF algorithms, and we observe that WW always improves the performance of RF, at the exception of few datasets for which the performance is similar. As for the randomized depth random forests \texttt{RD}$n$, we observe that their performance is generally inferior to that of WW and RF. This shows that the benefit brought by the weighted subtree aggregation mechanism cannot be simulated by averaging over trees of diverse depth.

When using default hyperparameters for all algorithms, we observe in Table~\ref{tab:fit_time} that the test AUC scores can decrease significantly for EGB libraries while RF algorithms seem more stable, and that there is no clear best performing algorithm in this case.

\begin{table}[htbp]
\caption{Training times (seconds) of all algorithms with their default hyperparameters (no hyperoptimization) on the 5 largest considered datasets and test AUC corresponding to these training times. Test AUC scores are worse than that of Table~\ref{tab:performance}, since no hyperoptimization is used. WW, which uses only 10 trees here (default number of trees), is generally among the fastest algorithms, for performances comparable to that of all baselines (bold is for best EGB training time or performance, underline for best \texttt{RF}$n$, \texttt{WW}$n$ or \texttt{RD}$n$ training time or performance). Standard deviations are reported in Table~\ref{tab:fit_time_std}.}
\small
\begin{tabular}{l|cccc|cc|cccc|cc}
\toprule
 & \multicolumn{6}{c|}{Training time (seconds)} & \multicolumn{6}{c}{Test AUC} \\
 \cmidrule(l){2-7} \cmidrule(r){8-13}
 &  \texttt{XGB} & \texttt{LGBM} & \texttt{CB} &  \texttt{HGB} & \texttt{RF} &  \texttt{WW} & \texttt{XGB} & \texttt{LGBM} & \texttt{CB} &  \texttt{HGB} & \texttt{RF} &  \texttt{WW} \\
\midrule
covtype  &   10 &   \bestb{3} &   120 &  14 &    21 & \bestf{3}  &  0.986 &  0.978 &  \bestb{0.989} &  0.960 &  \bestf{0.998} &  0.979 \\
higgs    &   36 &  \bestb{30} &   653 &  85 &  1389 & \bestf{179}  &  0.823 &  0.812 &  \bestb{0.840} &  0.812 &  \bestf{0.838} &  0.813 \\
internet &  9 &   \bestb{4} &   188 &   8 &    0.4 &  \bestf{0.3} &  \bestb{0.918} &  0.828 &  0.910 &  0.500 &  0.862 &  \bestf{0.889} \\
kddcup   &  175 &  41 &  2193 &  \bestb{31} &   208 &   \bestf{12}  &  \bestb{1.000} &  0.638 &  0.988 &  0.740 &  0.998 &  \bestf{1.000} \\
kick     &  7 &   \bestb{0.4} &  50 &  0.7 & 31 & \bestf{5}   &  0.768 &  0.757 &  \bestb{0.781} &  0.773 &  0.747 &  \bestf{0.751} \\
\bottomrule
\end{tabular}
\label{tab:fit_time}
\end{table}
\subsection{Training time}
We provide a study of training times over an increasing fraction of the same datasets in Figure~\ref{fig:train_times}.
We report also in Table~\ref{tab:fit_time} (see also Table~\ref{tab:fit_time_std} for standard deviations) the test AUC scores obtained with default hyperparameters of all algorithms on the 5 largest considered datasets together with their training times (timings can vary by several orders of magnitude with varying hyperparameters for EGB libraries, as observed by the timing differences between Figure~\ref{fig:auc_timings} and Table~\ref{tab:fit_time}).
\begin{figure}[htbp]
    \centering
    \includegraphics[width=\textwidth]{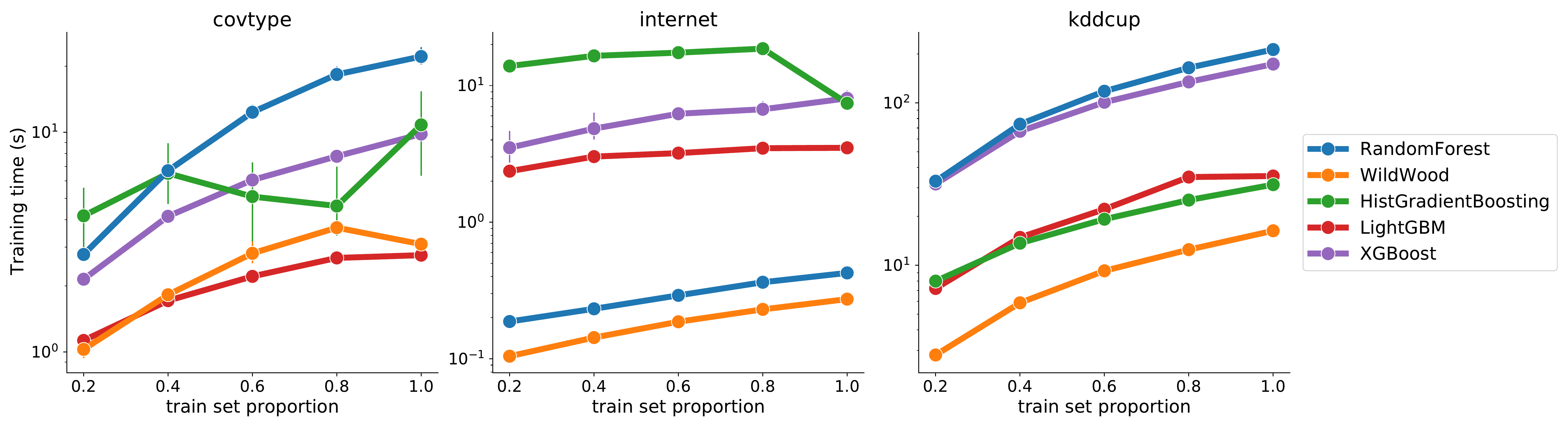}
    \caption{Training times on increasing fractions of a few large datasets for WildWood and the other baselines of Table~\ref{tab:fit_time} (except for CatBoost which is not competitive). The results are averaged over 5 runs using default parameters for each algorithm. WildWood clearly outpaces its competitors in almost all cases.}
    \label{fig:train_times}
\end{figure}
Figure~\ref{fig:train_times} further highlights the swiftness of WildWood's training procedure as compared to other models when trained on increasing fractions of the same dataset. 
Paradoxically, certain boosting algorithms appear to train faster on bigger data fractions, we attribute this to their stopping criteria on default parameters.

The results on both Tables~\ref{tab:performance} and~\ref{tab:fit_time} show that WW is competitive with respect to all baselines both in terms of performance and computational times: it manages to always reach at least comparable performance with the best algorithms despite only using $10$ trees as a default. 
In this respect, WW maintains high scores at a lower computational cost.

\subsection{Model size}
We also train random forest models with number of trees from one to ten (and default parameters for the rest) on a few datasets using random forest algorithms (WW and \texttt{scikit-learn}'s implementations of RF~\citep{louppe2014understanding} and Extra Trees~\citep{geurts2006extremely}). 
We plot the mean test AUC over 10 repetitions against model size measured through the associated pickle file in megabytes. The result is displayed on Figure~\ref{fig:modelsize_vs_perf} and demonstrates WildWood's ability to offer lightweight random forest models with improved performance. 
This is a valuable advantage for applications using learning algorithms with limited memory and computational resources. 
Notable examples include embedded systems~\citep{8972136, 8610075, 9046770} or IoT devices~\citep{yazici2018edge, domb2017lightweight, ren2020random}.

\begin{figure}[htbp]
    \centering
    \includegraphics[width=\textwidth]{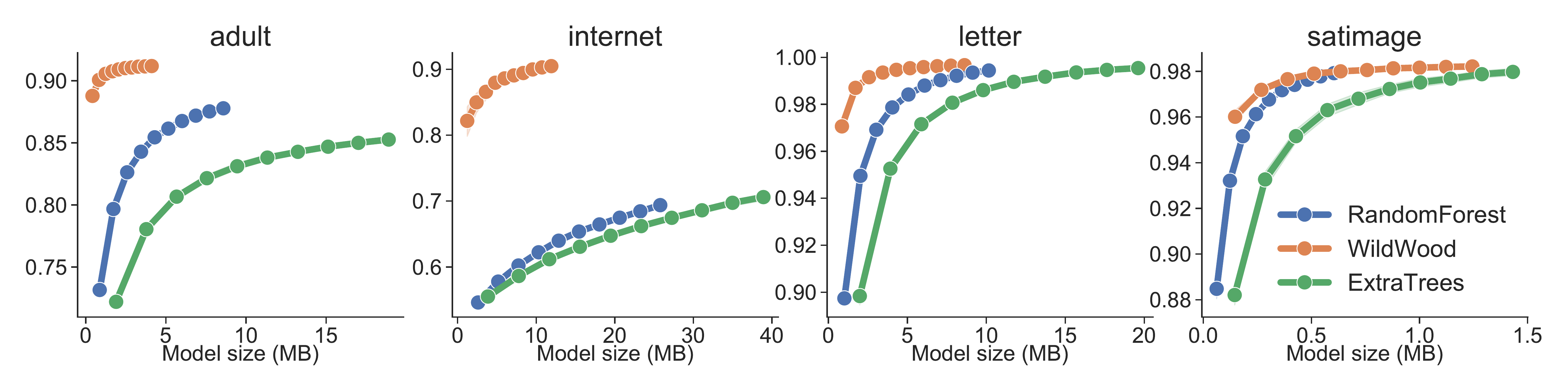}
    \caption{Mean test AUC and standard-deviations ($y$-axis) as a function of model size in megabytes ($x$-axis) using 10 train/test splits for WW and \texttt{scikit-learn}'s implementations of RF~\citep{louppe2014understanding} and Extra Trees~\citep{geurts2006extremely} on a few datasets. We use one to ten trees for each algorithm with default parameters. Wildwood is able to achieve better performance with smaller random forest models.}
    \label{fig:modelsize_vs_perf}
\end{figure}

\subsection{Regression experiment}
We run an elementary regression experiment comparing WW to RF and Extra Trees again. The task is to recover four noisy test signals: \textit{Doppler}, \textit{Heavisine}, \textit{Blocks} and \textit{Bumps} originally analyzed by~\cite{donoho1993nonlinear, donoho1994ideal}. We test different noise intensities measured through the signal-to-noise ratio (SNR) and show the results on Figure~\ref{fig:regression}. We observe that WildWood clearly outperforms its competitors on the signal recovery tasks at low SNR levels. This is thanks to the regularizing effect of aggregation making the algorithm more robust to noise. We note, nonetheless, that this may slightly degrade performance on relatively irregular signals at high SNR (\textit{Bumps}). 
\begin{figure}[htbp]
    \centering
    \includegraphics[width=\textwidth]{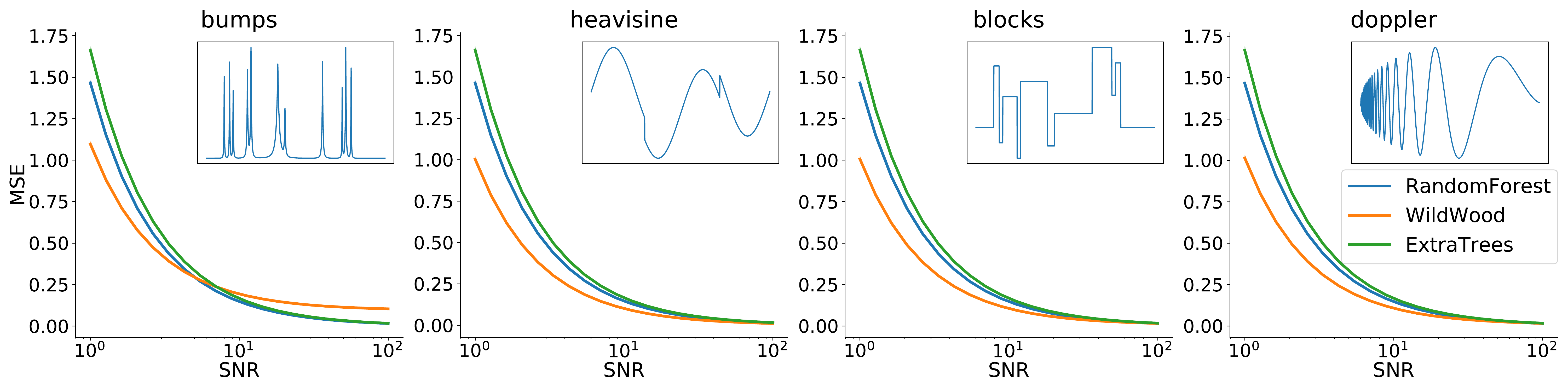}
    \caption{Averaged mean squared errors ($y$-axis) at increasing SNR levels ($x$-axis) over 10 repetitions of the noisy signal recovery task for WW and \texttt{scikit-learn}'s implementations of RF~\citep{louppe2014understanding} and Extra Trees~\citep{geurts2006extremely} on the four test signals \textit{Doppler}, \textit{Heavisine}, \textit{Blocks} and \textit{Bumps}. The noiseless signals are displayed in the inner frames. Each algorithm uses 100 trees and default values of the remaining parameters. Wildwood achieves smaller errors at low SNR levels thanks to its regularization effect.}
    \label{fig:regression}
\end{figure}

\subsection{Decision Boundaries}
Finally, we provide an additional illustration of the interpretability advantage of WW over RF similar to Figure~\ref{fig:aggregation_effect} but with real data. For this purpose, we plot the decision boundaries obtained on simple binary classification datasets for which good performance can be attained using only a pair of features. In each dataset, we choose such a pair among the most important features as measured by the Mean Decrease in Impurity (MDI) and fit WW and RF on the datasets restricted to these features. The results are shown on Figure~\ref{fig:aggregation_realdata}. As before, we observe that WW yields more regular decision boundaries making its model predictions less prone to overfitting and more interpretable.
\begin{figure}[htbp]
    \centering
    \includegraphics[width=\textwidth]{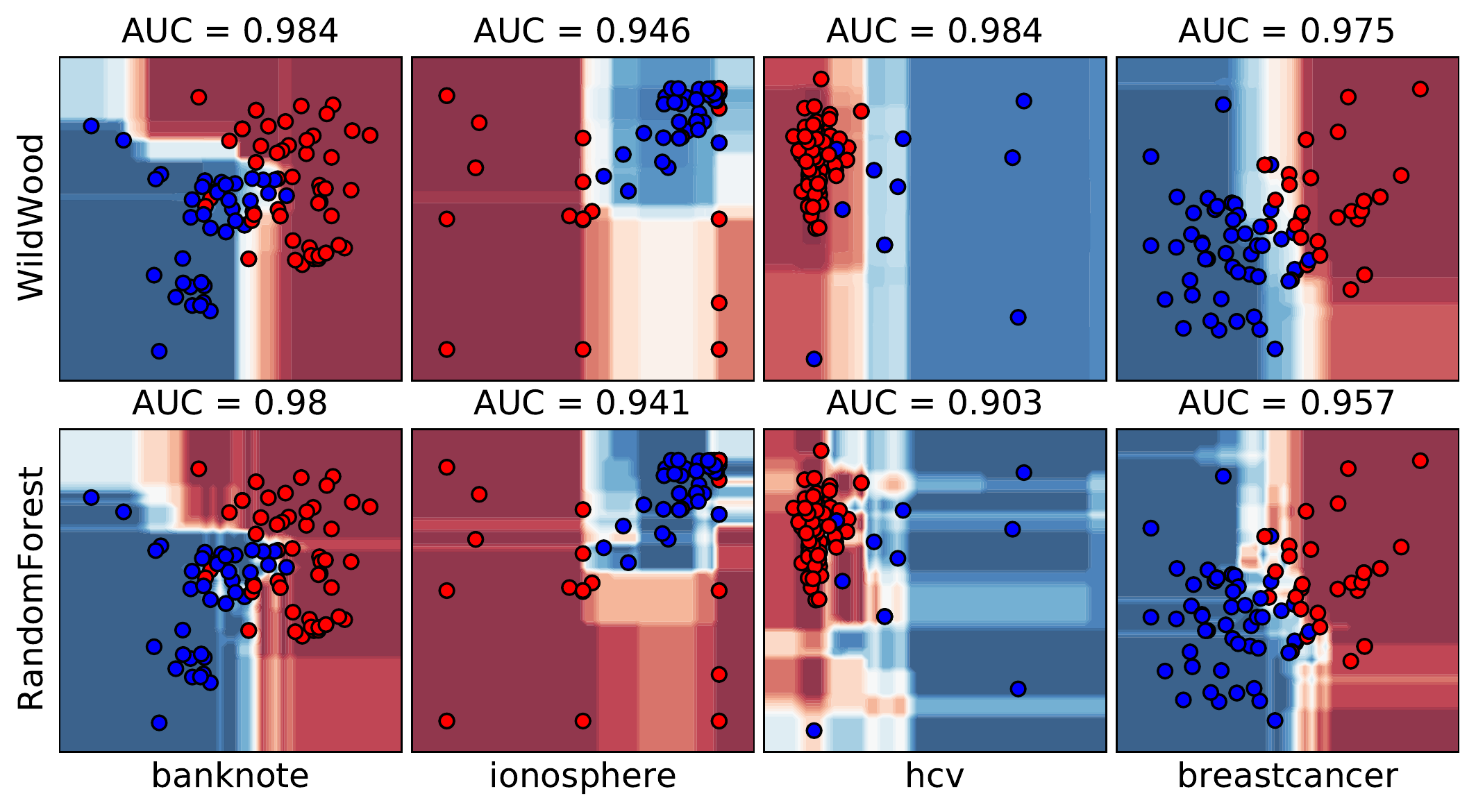}
    \caption{Illustration of the decision functions obtained with WW and RF using 10 trees each on simple binary classification datasets restricted to a pair of the most important features (according to MDI). The scattered samples represent the test set (colors indicate labels) while the decision boundaries use the train set. Thanks to aggregation, WW obtains more regular decision boundaries decreasing the risk of overfitting and improving interpretability..}
    \label{fig:aggregation_realdata}
\end{figure}

\section{Conclusion}
\label{sec:conclusion}

We introduced WildWood, a new Random Forest algorithm for batch supervised learning.
Tree predictions in WildWood are aggregation with exponential weights of the predictions of all subtrees, with weights computed on bootstrap out-of-bag samples. 
This leads to improved predictions in each individual tree, at a small computational cost, since WildWood's prediction complexity is similar to that of a standard Random Forest.
Moreover, thanks to the histogram strategy, WildWood's implementation is competitive with strong baselines including popular extreme boosting libraries, both in terms of performance and training times.
Note also that WildWood has few hyperparameters to tune and that the performances obtained with default hyperparameters are usually good enough in our experiments.

WildWood's implementation is still evolving and many improvements coming with future updates are planned, including the computation of feature importance, GPU training, distributed training (we only support single-machine training for now), among other enhancements that will further improve performances and accelerate computations.
Room for improvement in WildWood comes from the fact that the overall forest prediction is a simple arithmetic mean of each tree prediction, while we could perform also exponentially weighted aggregation between trees.
Future works include a WildWood-based implementation of isolation-forest~\citep{4781136}, using the same subtrees aggregation mechanism with the log loss for density estimation, to propose a new algorithm for outliers detection.

\paragraph{Acknowledgments.}

This research is supported by the Agence Nationale de la Recherche as part of the ``Investissements d'avenir'' program (reference ANR-19-P3IA-0001; PRAIRIE 3IA Institute).
Yiyang Yu is supported by grants from Région Ile-de-France.

\newpage
\section{Proof of Theorem~\ref{thm:ctw} and construction of Algorithms~\ref{alg:wbar-computation} and~\ref{alg:tree-prediction}}
\label{sec:proof-thm-ctw}

The expression in Equation~\eqref{eq:exact-aggregation} involves sums over all subtrees $T$ of the fully grown tree $\tree$ (involving an exponential in the number of leaves of $\tree$).
However, it can be computed efficiently because of the specific choice of the prior $\pi$.
More precisely, we will use the following lemma \citep[Lemma~1]{helmbold1997pruning} several times to efficiently compute sums of products.
Let us recall that $\nodes(\tree)$ stands for the set of nodes of~$\tree$.
\begin{lemma}
  \label{lem:ctw-sum-prod}
  Let $g: \nodes(\tree) \to \R$ be an arbitrary function and define $G: \nodes(\tree) \to \R$ as
  \begin{equation}
  \label{eq:ctw-sum-prod-def}
  G (\node) = \sum_{T \subset \tree_\node} 2^{- \| T \|} 
  \prod_{\node' \in \leaves(T)} g(\node'),
  \end{equation}
  where the sum over $T \subset \tree_\node$ means the sum over all subtrees $T$ of $\tree$ rooted at $\node$. 
  Then\textup, $G (\node)$ can be computed recursively as follows\textup:
  \begin{equation*}
    G(\node) = 
    \begin{cases}
      g(\node) & \text{ if } \node \in \leaves(\tree) \\
      \frac{1}{2} g(\node) + \frac{1}{2} G(\node 0) G(\node 1) & \mathrm{ otherwise,}
    \end{cases}
  \end{equation*}
  for each node $\node \in \nodes(\tree)$.
\end{lemma}
For the sake of completeness, we include a proof of this statement.
\begin{proof}
First, let us notice that the case $\node \in \leaves(\tree)$ is straightforward since there is only one pruning $T$ of $\tree_\node$ which satisfies $\|T\| = 0$ (recall that $\|T\|$ is the number of internal nodes and leaves in $T$ \textit{minus} the number of leaves in $T$ that are also leaves of $\tree_\node$).
For the second case, we can expand $G(\node)$ by taking into account the pruning which only leaves $\node$ as a leaf, the rest of the prunings can be expressed through pairs of prunings $T_0$ and $T_1$ of $\tree_{\node 0}$ and $\tree_{\node 1}$ respectively. Moreover, it can be shown that such a pruning $T$ satisfies $\|T\| = 1 + \|T_0\| + \|T_1\|$, thus we get the following expansion :
\begin{align*}
    G(\node) &= \frac{1}{2}g(\node) + \sum_{T_0 \subset \tree_{\node 0}}\sum_{T_1 \subset \tree_{\node 1}} 2^{-(1 + \|T_0\| + \|T_1\|) }\prod_{\node' \in T_0} g(\node 0 \node') \prod_{\node{''} \in T_1} g(\node 1 \node{''}) \\
    &= \frac{1}{2}g(\node) + \frac{1}{2} \Big(\sum_{T_0 \subset \tree_{\node 0}} 2^{-\|T_0\| }\prod_{\node' \in T_0} g(\node 0 \node') \Big) \cdot \Big(\sum_{T_1 \subset \tree_{\node 1}} 2^{-\|T_1\| }\prod_{\node{''} \in T_1} g(\node 1 \node{''}) \Big) \\
    &= \frac{1}{2}g(\node) + \frac{1}{2} G(\node 0) G(\node 1).
\end{align*}
This concludes the proof of Lemma~\ref{lem:ctw-sum-prod}.
\end{proof}
Let us introduce $w_T = \pi(T) \exp(- \eta L_T)$ for any $T \subset \tree$,
so that Equation~\eqref{eq:exact-aggregation} writes
\begin{equation}
\label{eq:ewa-tree}
\wh f(x) = \frac{\sum_{T \subset \tree} w_T \pred_{T} (x)}{\sum_{T \subset \tree} w_T},
\end{equation}
where the sums hold over all the subtrees $T$ of $\tree$ rooted at $\root$ (the root of the full tree $\tree$).
We will show how to efficiently compute and update the numerator and denominator in Equation~\eqref{eq:ewa-tree}.
Note that $w_T$ may be written as
\begin{align}
    \nonumber
    w_T &= \pi(T) \exp(- \eta L_T) \\
    \nonumber
    &= 2^{-\| T \|} \exp \bigg(- \eta \sum_{i \in I_\otb} \ell (\pred_{\node_T(x_i)}, y_i) \bigg) \\
    \label{eq:weight-tree-decomp-1} 
    &= 2^{-\| T \|} \exp \bigg(- \eta \sum_{\node \in \leaves(T)}
    \sum_{i \in I_\otb \pp x_i \in C_\node} \ell (\pred_{\node_T(x_i)}, y_i) \bigg) \\
    \label{eq:weight-tree-decomp-2}
    &= 2^{-\| T \|} \exp \bigg(- \eta \sum_{\node \in \leaves(T)}
    \sum_{i \in I_\otb \pp x_i \in C_\node} \ell (\pred_{\node}, y_i) \bigg) \\
    \nonumber
    &= 2^{-\| T \|} \exp \bigg(- \eta \sum_{\node \in \leaves(T)} L_\node \bigg) \\
    \label{eq:weight-tree-decomp-3}
    &= 2^{-\| T \|} \prod_{\node \in \leaves(T)} w_{\node},
\end{align}
where we recall that
\begin{equation*}
 L_\node = \sum_{i \in I_\otb \pp x_i \in C_\node} \ell (\pred_{\node}, y_i)
 \quad \text{and} \quad w_\node = \exp(-\eta L_\node).
\end{equation*}
Equality~\eqref{eq:weight-tree-decomp-1} comes from the fact that the set of cells $\{C_\node : \node \in \leaves(T) \}$ is a partition of $C$ by construction, and that the stopping criterion used to build $\tree$ ensures that each leaf node in $\leaves(T)$ contains at least one sample from $I_\otb$ (see Section~\ref{sub:splits}).
Equality~\eqref{eq:weight-tree-decomp-2} comes from the fact that the prediction of a node is constant and equal to $\pred_\node$ for any $x \in C_\node$.

\paragraph{Denominator of Equation~\eqref{eq:ewa-tree}.}

For each node $\node \in \nodes(\tree)$, denote
\begin{equation}
\label{eq:avg-weights-def}
\wbar_{\node} = \sum_{T \subset \tree_\node} 2^{-\| T \|} \prod_{\node' \in \leaves (T)} w_{\node'},
\end{equation}
where once again the sum over $T \subset \tree_\node$ means the sum over all subtrees $T$ of $\tree$ rooted at $\node$.
We have that~\eqref{eq:weight-tree-decomp-3} entails
\begin{equation}
\label{eq:wbar-root}
\wbar_{\root} = \sum_{T \subset \tree_\root} 2^{-\| T \|} \prod_{\node \in \leaves (T)} w_{\node} = \sum_{T  \subset \tree_\root } w_T = \sum_{T  \subset \tree} w_T \, .
\end{equation}
So, we can compute recursively $\wbar_{\root}$ very efficiently, using a recursion on the weights $\wbar_{\node}$ using Lemma~\ref{lem:ctw-sum-prod} with $g(\node) = w_\node$.
This leads to the recursion stated in Theorem~\ref{thm:ctw}, see Equation~\eqref{eq:wden-recursion}. 

Now, we can exploit the fact that decision trees are built in a depth-first fashion in WildWood: all the nodes $\node \in \tree$ are stored in a ``flat'' array, and by construction both the child nodes $\node 0$ and $\node 1$ have indexes that are larger than the one of $\node$.
So, we can simply loop over the array of nodes in reverse order, and compute $\wbar_{\node} = w_\node$ if $\node \in \leaves(\tree)$ and  $\wbar_{\node} = \frac{1}{2} w_{\node} + \frac{1}{2} \wbar_{\node 0} \wbar_{\node 1}$ otherwise: we are guaranteed to have computed $\wbar_{\node 0}$ and $\wbar_{\node 1}$ before computing $\wbar_{\node}$.
This algorithm is described in Algorithm~\ref{alg:wbar-computation}.
Since these computations involve a large number of products with exponentiated numbers, it typically leads to strong over- and under-flows: we describe in Algorithm~\ref{alg:wbar-computation} a version of this algorithm which works recursively over the logarithms of the weights.
At the end of this loop, we end up at $\node = \root$ and have computed $\wbar_{\root} = \sum_{T  \subset \tree} w_T$ with a very efficient $O(|\nodes(\tree)|)$ complexity.
Note also that it is sufficient to store both $w_\node$ and $\wbar_\node$ for all $\node \in \tree$, which makes for a $O(|\nodes(\tree)|)$ memory consumption.

\paragraph{Numerator of Equation~\eqref{eq:ewa-tree}.}

The numerator of Equation~\eqref{eq:ewa-tree} almost follows the exact same argument as the denominator, but since it depends on the input vector $x \in \cell$ of features for which we want to produce a prediction, it is performed at inference time.
Recall that $\pathpoint(x)$ is the sequence of nodes that leads to the leaf $\node_{\tree}(x)$ containing $x \in \cell$ and define, for any $\node \in \nodes(\tree)$, $\wpred_{\node}(x) = w_\node \pred_\node(x)$ if $\node \in \pathpoint(x)$, and $\wpred_{\node}(x) = w_\node$ otherwise.
We have
\begin{align}
    \nonumber
    \sum_{T \subset \tree} w_T \pred_{T} (x) &= \sum_{T \subset \tree_\root} w_{T} \pred_{\node_T(x)} \\
    \label{eq:denom-argument-1}
    &=  \sum_{T \subset \tree_\root} 2^{-\| T \|} \prod_{\node \in \leaves(T)} w_{\node} \pred_{\node_T(x)} \\
    \label{eq:denom-argument-2}
    &= \sum_{T \subset \tree_\root} 2^{-\| T \|} \prod_{\node \in \leaves(T)} \wpred_{\node}(x).
\end{align}
Note that~\eqref{eq:denom-argument-1} comes from~\eqref{eq:weight-tree-decomp-3} while~\eqref{eq:denom-argument-2} comes from the definition of $\wpred_{\node}(x)$ (note that a single term from the product over $\node \in \leaves(T)$ corresponds to $\node = \node_T(x)$ since $\{ C_\node : \node \in \leaves(T) \}$ is a partition of $\cell$).
We are now in position to use again Lemma~\ref{lem:ctw-sum-prod} with $g(\node) = \wpred_{\node}(x)$.
Defining
\begin{equation*}
\wnum_\node(x)
= \sum_{T \subset \tree_\node} 2^{-\| T \|}
\prod_{\node' \in \leaves(T)} \wpred_{\node'}(x) \, ,
\end{equation*}
we can conclude that 
\begin{equation}
    \label{eq:wnum-root}
    \wnum_\root(x) = \sum_{T \subset \tree} w_T \pred_{T} (x)
\end{equation}
and that the following recurrence holds:
\begin{equation}
  \label{eq:wnum-recursion1}
  \wnum_{\node}(x) =
  \begin{cases}
      \wpred_{\node}(x) & \text{ if } \node \in \leaves(\tree) \\
      \frac{1}{2} \wpred_{\node}(x) + \frac{1}{2} \wnum_{\node 0}(x) \wnum_{\node 1}(x) &\text{ otherwise}.
  \end{cases}
\end{equation}
This recurrence allows to compute $\wnum_{\node}(x)$ from $\wpred_{\node}(x)$, but note that a direct use of this formula would lead to a complexity $O(|\nodes(\tree)|)$ to produce a prediction for a single input $x \in \cell$. 
It turns out can we can do much better than that.

Indeed, whenever $\node \notin \pathpoint(x)$, we have by definition that $\wpred_{\node}(x) = w_{\node}$ and that $\wpred_{\node'}(x) = w_{\node'}$ for any descendant $\node'$ of $\node$, which entails by induction that $\wnum_{\node}(x) = \wbar_\node$ for any $\node \notin \pathpoint(x)$.
Therefore, we only need to explain how to compute $\wnum_{\node}(x)$ for $\node \in \pathpoint(x)$.
This is achieved recursively, thanks to~\eqref{eq:wnum-recursion1}, starting at the leaf $\node_\tree(x)$ and going up in the tree to $\root$:
\begin{equation}
    \label{eq:wnum-recursion2}
  \wnum_{\node}(x) =
  \begin{cases}
      w_{\node} \pred_{\node} & \text{ if } \node = \node_\tree (x) \\
      \frac{1}{2} w_{\node} \pred_{\node} + \frac{1}{2} \wbar_{\node (1-a)} \wnum_{\node a}(x) &\text{ otherwise, where } a \in \{ 0, 1 \} \text{ is s.t. } \node a \in \pathpoint(x).
  \end{cases}
\end{equation}
Let us explain where this comes from: firstly, one has obviously that $\leaves(\tree) \cap \pathpoint(x) = \{ \node_\tree(x) \}$, so that $\wnum_{\node}(x) = g(\node) = \wpred_{\node}(x) = w_\node \pred_\node(x)$ for $\node = \node_\tree (x)$.
Secondly, we go up in the tree along $\pathpoint(x)$ and use again~\eqref{eq:wnum-recursion1}: whenever $\node \in \inodes(\tree)$ and $\node a \in \pathpoint(x)$ for $a \in \{0, 1\}$, we have $\wnum_{\node (1 - a)}(x) = \wbar_{\node (1 - a)}$ since $\node (1 - a) \notin \pathpoint(x)$. 
This recursion has a complexity $O(|\pathpoint(x)|)$ where $|\pathpoint(x)|$ is the number of nodes in $\pathpoint(x)$, and is typically orders of magnitude smaller than $|\nodes(\tree)|$ (in a well-balanced binary tree, one has the relation $|\pathpoint(x)| = O(\log_2(|\nodes(\tree)|))$).
Moreover, we observe that the recursions used in~\eqref{eq:wden-recursion} and~\eqref{eq:wnum-recursion2} only need to save both $w_\node$ and $\wbar_\node$ for any $\node \in \nodes(\tree)$.

Finally, we have using~\eqref{eq:wbar-root} and~\eqref{eq:wnum-root} that
\begin{equation*}
    \wh f(x) = \frac{\sum_{T \subset \tree} w_T \pred_{T} (x)}{\sum_{T \subset \tree} w_T} = \frac{\wnum_\root(x)}{\wbar_\root} =: \wh f_{\root}(x),
\end{equation*}
and we want to compute $\wh f_{\root}(x)$ recursively from $\wh f_{\node}(x)$ where $\node \in \pathpoint(x)$.
First, whenever $\node = \node_\tree (x)$ we have
\begin{equation*}
    \wh f_\node(x) = \frac{\wnum_\node(x)}{\wbar_\node} = \frac{w_{\node} \pred_{\node}}{w_{\node}} = \pred_{\node},
\end{equation*}
while for $\node \neq  \node_\tree (x)$ and $\node \in \pathpoint(x)$, we write
\begin{align}
    \label{eq:fnode-computation1}
    \wh f_\node(x) = \frac{\wnum_\node(x)}{\wbar_\node}
    &= \frac{\frac{1}{2} w_{\node} \pred_{\node} + \frac{1}{2} \wbar_{\node (1-a)} \wnum_{\node a}(x)}{\wbar_\node} \\
    \label{eq:fnode-computation2}
    &= \frac 12 \frac{w_{\node}}{\wbar_\node} \pred_{\node} 
    + \frac 12 \frac{\wbar_{\node (1-a)} \wbar_{\node a}} {\wbar_\node} \frac{\wnum_{\node a}(x)}{\wbar_{\node a}} \\
    \label{eq:fnode-computation3}
    &= \frac 12 \frac{w_{\node}}{\wbar_\node} \pred_{\node} + \Big(1 - \frac 12  \frac{w_{\node}}{\wbar_\node} \Big) \wh f_{\node a}(x),
\end{align}
where~\eqref{eq:fnode-computation1} comes from~\eqref{eq:wnum-recursion2} while \eqref{eq:fnode-computation3} comes from~\eqref{eq:wden-recursion}.
This proves the recursion stated in Equation~\eqref{eq:f_pred_recursion} from Theorem~\ref{thm:ctw}, and to Algorithm~\ref{alg:tree-prediction}.
This concludes the proof of Theorem~\ref{thm:ctw}. $\hfill \square$

\section{Proofs of the results from Section~\ref{sec:theory}}
\label{sec:proofs-oracle}

The proof of Theorem~\ref{thm:subtree-bound} is partly inspired from the proof of Theorem~2 in~\cite{dalalyan2008aggregation}, that we generalize to exp-concave losses, while only least-squares regression is considered therein.
Let $E$ be a measurable space and $P, Q$ be probability measures on it.
The Kullback-Leibler divergence between $P$ and $Q$ is defined by
\begin{equation*}
    \kl(P, Q) = \int_E \log \Big( \frac{\d P}{\d Q} \Big) \d P
\end{equation*}
whenever $P$ is absolutely continuous with respect to $Q$ and equal to $+\infty$ otherwise.
Also, if $h : E \rightarrow \R$ is a measurable function such that $\int_E h \d P$ is well-defined on $\R \cup \{ -\infty, +\infty \}$, we introduce
\begin{equation*}
    P_h := \frac{h}{\int_E h \d P} \cdot P,
\end{equation*}
the probability measure on $E$ with density $h / \int h \d P$ with respect to $P$.
A classical result is the Donsker-Varadhan variational formula~\citep{donsker1976asymptotic}, which is at the core of the proofs of many PAC-Bayesian theorems~\citep{mcallester1999pacbayes,catoni2007pacbayes} and that we use here as well in the proof of Theorem~\ref{thm:subtree-bound}.
It states that
\begin{equation}
    \label{eq:donsker-varadhan}
    \log \Big( \int_E \exp(h) \d Q \Big) + \kl(P, Q) - \int h \d P = \kl(P, Q_{\exp(h)})
\end{equation}
holds for any probability measures $P$ and $Q$ on $E$ and any measurable function $h : E \rightarrow \R$.
This entails in particular that
\begin{equation*}
    \log \Big( \int_E \exp(h) \d Q \Big) = \sup_{P} \Big\{ \int h \d P - \kl(P, Q) \Big\},
\end{equation*}
where the supremum is over all probability measures on $E$, and where the supremum is achieved for $P = Q_{\exp(h)}$ whenever the term on the left-hand side is finite.
\begin{proof}[Proof of Theorem~\ref{thm:subtree-bound}]
Recall that the \otb loss of a subtree $T \subset \tree$ is given by
\begin{equation*}
    L_T = \sum_{i \in I_\otb} \loss (\pred_{T} (x_i), y_i)
\end{equation*}
and let us introduce
\begin{equation}
    \label{eq:p_T_def}
    p_T = \frac{\pi(T)\exp(-\eta L_T)}{\sum_{T'} \pi(T')\exp(-\eta L_{T'})}
\end{equation}
for any subtree $T \subset \tree$.
First, we use the fact that $\ell$ is a $\eta$-exp-concave loss function, hence $\eta$-mixable, see Section~3.3 from~\cite{cesabianchi2006plg}, which entails, since $p_T$ is a probability measure over the set of all subtrees $T \subset \tree$, that
\begin{equation*}
   \loss \Big(\sum_T p_T \wh y_T(x_i), y_i \Big) \leq -\frac{1}{\eta} \log \Big( \sum_T p_T \exp(-\eta \loss(\wh y_T(x_i), y_i)) \Big),
\end{equation*}
where the sums over $T$ are over all subtrees $T \subset \tree$.
Now, summing this inequality over $i \in I_\otb$ and using the convexity of the log-sum-exp function leads to
\begin{align*}
\sum_{i \in I_\otb} \loss \Big(\sum_T p_T \wh y_T(x_i), y_i \Big) &\leq -\frac{1}{\eta} \sum_{i \in I_\otb} \log \Big( \sum_T p_T \exp(-\eta \loss(\wh y_T(x_i), y_i)) \Big) \\
&\leq -\frac{n_{\otb}}{\eta} \log \Big( \sum_T p_T \exp \Big(- \frac{\eta}{n_{\otb}} \sum_{i \in I_\otb} \loss(\wh y_T(x_i), y_i) \Big) \Big) \\
&= -\frac{n_{\otb}}{\eta} \log \Big( \sum_T p_T \exp \Big(- \frac{\eta}{n_{\otb}} L_T \Big) \Big).
\end{align*}
By plugging the definition of $p_T$ into the previous expression, and by introducing $\rho(T) := \eta L_T / n_\otb$, we obtain
\begin{align*}
\frac{S}{n_\otb} &:= \frac{1}{n_\otb} \sum_{i \in I_\otb} \loss(\wh f(x_i), y_i) \\
&\leq - \frac{1}{\eta} \log \Big( \sum_T \pi(T) \exp \big( -(n_\otb + 1) \rho(T) \big) \Big) + \frac{1}{\eta} \log \Big( \sum_T \pi(T) \exp \big( -n_\otb \rho(T) \big) \Big).
\end{align*}
The H\"older inequality implies that
\begin{equation*}
    \sum_T \pi(T) \exp \big( -n_\otb \rho(T) \big) \leq \Big(\sum_T \pi(T) \exp \big( -(n_\otb + 1) \rho(T) \big) \Big)^{n_\otb / (n_\otb + 1)},
\end{equation*}
thus
\begin{equation*}
   \frac{S}{n_\otb} \leq - \frac{1}{\eta(n_\otb + 1)} \log \Big( \sum_T \pi(T) \exp \big( -(n_\otb + 1) \rho(T)\big) \Big).
\end{equation*}
Using~\eqref{eq:donsker-varadhan} with $h(T) = -(n_\otb + 1) \rho(T)$ and $Q = \pi$, we have 
\begin{align*}
\log \Big( \sum_T \pi(T) \exp &\big(-(n_\otb + 1) \rho(T) \big) \Big) \\
&= - \sum_T P(T) (n_\otb + 1) \rho(T) - \kl(P, \pi) + \kl(P, \pi_{\exp(h)})
\end{align*} 
for any probability measure $P$ over the set of subtrees of $\tree$.
This leads to
\begin{equation*}
    \frac{1}{n_\otb} \sum_{i \in I_\otb} \loss(\wh f(x_i), y_i) \leq \frac{1}{n_\otb} \sum_T P(T) L_T + \frac{1}{\eta(n_\otb + 1)} \kl(P, \pi) 
\end{equation*}
for any $P$, since $\kl(P, \pi_{\exp(h)}) \geq 0$. So for the particular choice $P = \delta_T$ (the Dirac mass at $T$) for any subtree $T \subset \tree$, we have
\begin{align*}
\frac{1}{n_\otb} \sum_{i \in I_\otb} \loss(\wh f(x_i), y_i) &\leq 
\frac{1}{n_\otb} L_T + \frac{1}{\eta (n_\otb + 1)} \log\big( \pi(T)^{-1} \big) \\
&\leq \frac{1}{n_\otb} L_T + \frac{\log 2}{\eta} \frac{\|T\|}{n_\otb + 1},
\end{align*}
which concludes the proof of Theorem~\ref{thm:subtree-bound}.
\end{proof}
The proof of Corollary~\ref{cor:bound-best-pruning-log} requires the next Lemma.
\begin{lemma}
\label{lem:analog-lem5}
Consider classification with $\cY = \{ 1, \ldots, K \}$ and a node $\node \in \nodes(\tree)$. 
Denote $n_{\node}(k)$ the number of samples of class $k$ in node $\node$. 
We consider the Krichevsky-Trofimov estimator 
\begin{equation*}
    \widehat{y}(k) = \frac{n_{\node}(k) + 1/2}{n_{\node} + K/2}   
\end{equation*}
where $n_{\node} = \sum_{k=1}^K n_{\node}(k)$ and the log loss $\ell(y', y) = -\log y'(y)$.
Then, the inequality
\begin{equation*}
    \sum_i \ell(\widehat{y}, y_i) - \inf_p \sum_i \ell(p, y_i) \leq \frac{K-1}{2} 
\end{equation*}
holds.
\end{lemma}
\begin{proof}
We know that the optimal $p$ is given by $p_k = n_{\node}(k) / n_{\node}$. 
Indeed, it is the solution to the following constrained convex optimization problem
\begin{align*}
    \min_p \; -\sum_{k=1}^K n_{\node}(k) \log p_k \quad \text{ subject to } \quad \sum_{k=1}^K p_k = 1,
\end{align*}
where we consider non-negativity to be already enforced by the objective and imposing $p_k \leq 1$ is redundant with the constraint $\sum_{k=1}^K p_k = 1$.
We can write the Lagrangian function as
\begin{equation*}
  L(p, \lambda) = -\sum_{k=1}^K n_{\node}(k) \log p_k + \lambda \Big( \sum_{k=1}^K p_k - 1 \Big)
\end{equation*}
and one can check the KKT conditions when taking $p_k = n_{\node}(k) / n_{\node}$ and $\lambda = n_{\node}$. 
Since we are dealing with a convex problem with linear constraints, this is a sufficient optimality condition.
Straightforward computations give
\begin{align*}
\sum_i \ell(\widehat{y}, y_i) - \inf_p \sum_i \ell(p, y_i) &= \sum_{k=1}^K - n_{\node}(k) \log(\widehat{y}(k)) - \sum_{k=1}^K - n_{\node}(k) \log p_k \\
&=\sum_{k=1}^K - n_{\node}(k) \log\frac{n_{\node}(k) + 1/2}{n_{\node} + K/2} - \sum_{k=1}^K - n_{\node}(k) \log\frac{n_{\node}(k)}{n_{\node}} \\
&=\sum_{k=1}^K - n_{\node}(k) \Big( \log\frac{n_{\node}}{n_{\node} + K/2} + \log\frac{n_{\node}(k) + 1/2}{n_{\node}} - \log\frac{n_{\node}(k)}{n_{\node}} \Big) \\
&= - n_{\node} \log\frac{n_{\node}}{n_{\node} + K/2} + \sum_{k=1}^K - n_{\node}(k) \Big( \log\frac{n_{\node}(k) + 1/2}{n_{\node}} -  \log\frac{n_{\node}(k)}{n_{\node}} \Big) \\
&= n_{\node} \log\frac{n_{\node} + K/2}{n_{\node}} + \sum_{k=1}^K n_{\node}(k) \log\frac{n_{\node}(k)}{n_{\node}(k) + 1/2}.
\end{align*}
Now, using the concavity of the logarithm gives
\begin{equation*}
\sum_i \ell(\widehat{y}, y_i) - \inf_p \sum_i \ell(p, y_i) \leq n_{\node} \log\frac{n_{\node} + K/2}{n_{\node}} + n_{\node} \log \Big( \sum_{k=1}^K \frac{n_{\node}(k)}{n_{\node} } \frac{n_{\node}(k)}{n_{\node}(k) + 1/2} \Big),
\end{equation*}
and the fact that $x \mapsto x / (x + 1/2)$ is non-decreasing and $n_\node(k) \leq n_\node$ leads to
\begin{align*}
\sum_i \ell(\widehat{y}, y_i) - \inf_p \sum_i \ell(p, y_i) &\leq n_{\node} \log\frac{n_{\node} + K/2}{n_{\node}} + n_{\node} \log \frac{n_{\node}}{n_{\node} + 1/2} \\
&= n_{\node} \log\frac{n_{\node} + 1/2 + (K-1)/2}{n_{\node} + 1/2} \\
&= n_{\node} \log\Big(1 + \frac{K-1}{2 n_{\node} + 1} \Big) \leq \frac{K-1}{2}.
\end{align*}
This concludes the proof of Lemma~\ref{lem:analog-lem5}.
\end{proof}

\begin{proof}[Proof of Corollary~\ref{cor:bound-best-pruning-log}]
The log-loss is trivially $1$-exp-concave, so that we can choose $\eta = 1$.
Following Theorem~\ref{thm:subtree-bound}, it remains to bound the regret of the tree forecaster $T$ with respect to the optimal labeling of its leaves.
For classification and the log loss, we use Lemma~\ref{lem:analog-lem5} to obtain
\begin{equation*}
   \sum_{i \in I_\otb: x_i \in \cell_\leaf} \loss(\pred_T(x_i), y_i) - \inf_p \sum_{i \in I_\otb: x_i \in \cell_\leaf} \loss(p, y_i) \leq \frac{K - 1}{2}
\end{equation*}
for any subtree $T$ and any $\leaf \in \leaves(T)$.
Now, summing over $\leaf \in \leaves(T)$, of cardinality $(\|T\| + 1) / 2$, leads to
\begin{equation*}
   \sum_{i \in I_\otb} \loss(\widehat f (x_i), y_i) - \sum_{i \in I_\otb} \loss(g_T(x_i), y_i) \leq \|T\| \log 2 + \frac{(K-1)(\|T\| + 1)}{4},
\end{equation*}
which concludes the proof of Corollary~\ref{cor:bound-best-pruning-log}.
\end{proof}

\begin{proof}[Proof of Corollary~\ref{cor:bound-best-pruning-square}]
The square loss is $1/(8B^2)$-exp-concave on $[-B,B]$, see~\cite{cesabianchi2006plg}, so we can choose $\eta = 1/(8B^2)$.
Following Theorem~\ref{thm:subtree-bound},
it remains to bound the regret of the tree forecaster $T$ with respect to the optimal labeling of its leaves.
For regression with the least-squares loss, since we use the empirical mean forecaster~\eqref{eq:reg-predictor}, we have
\begin{equation*}
   \sum_{i \in I_\otb: x_i \in \cell_\leaf} \loss(\pred_T(x_i), y_i) - \inf_b \sum_{i \in I_\otb: x_i \in \cell_\leaf} \loss(b, y_i) = 0
\end{equation*}
for any subtree $T$ and any leaf $\leaf \in \leaves(T)$.
The rest of the proof follows that of Corollary~\ref{cor:bound-best-pruning-log}.
\end{proof}

\section{Supplementary details on experiments}

We report in Table~\ref{tab:performance-with-std} the same test AUC scores as in Table~\ref{tab:performance} of all algorithms after hyperoptimization on the considered datasets. Standard-deviations displayed between parentheses are computed from 5 trainings with different random seeds. We observe that WW has better (or identical in some cases) performances than RF on all datasets and that it is close to that of EGB libraries (bold is for best EGB performance, underline for best \texttt{RF}$n$, \texttt{WW}$n$ or \texttt{RD}$n$ performance). Table~\ref{tab:performance-logloss} displays the results of the same experiment measured using the log loss.

We report also in Table~\ref{tab:fit_time_std} the same training time and test AUC as in Table~\ref{tab:fit_time}, with standard-deviations displayed between parentheses computed from 5 trainings with different random seeds, all with default hyperparameters of each algorithm. 
We observe that WW is generally among the fastest algorithms, for performances comparable to ones of all baselines (bold is for best EGB training time or performance, underline for best \texttt{RF} or \texttt{WW} training time or performance).

\begin{table}[htbp]
\caption{The same test AUC scores as in Table~\ref{tab:performance} of all algorithms after hyperoptimization on the considered datasets. Standard-deviations displayed between parentheses are computed from 5 trainings with different random seeds. We observe that WW has better (or identical in some cases) performances than RF on all datasets and that it is close to that of EGB libraries (bold is for best EGB performance, underline for best \texttt{RF}$n$, \texttt{WW}$n$ or \texttt{RD}$n$ performance).
}
\centering
\footnotesize
\begin{sideways}
\resizebox{0.9\textheight}{!}{\begin{tabular}{l|cccc|cccccc}
\toprule
   &  \texttt{XGB} & \texttt{LGBM} & \texttt{CB} &    \texttt{HGB} & \texttt{RF}$10$ &  \texttt{RF}$100$ & \texttt{WW}$10$ & \texttt{WW}$100$ & \texttt{RD}$10$ & \texttt{RD}$100$ \\
\midrule
adult        &  0.930 (2.7e-04) &  \bestb{0.931} (1.2e-04) &  0.927 (2.9e-04) &  0.930 (2.9e-04) &  0.915(5.0e-04) & 0.918(1.7e-04) & 0.916(4.1e-04) & \bestf{0.919}(1.4e-04) & 0.915(3.4e-04) & 0.917(1.5e-04) \\
bank         &  0.933 (1.5e-04) &  \bestb{0.935} (4.1e-05) &  0.925 (6.5e-04) &  0.930 (7.4e-04) &  0.919(6.0e-04) & 0.929(2.1e-04) & 0.926(3.8e-04) & \bestf{0.931}(1.6e-04) & 0.919(9.0e-04) & 0.922(2.2e-04) \\
breastcancer &  0.991 (4.4e-04) &  0.993 (1.1e-04) &  0.987 (6.7e-03) &  \bestb{0.994} (0.0e+00) &  0.987(2.1e-03) & \bestf{0.992}(4.1e-04) & 0.989(2.6e-03) & \bestf{0.992}(3.1e-04) & 0.987(2.9e-03) & 0.985(8.3e-04)  \\
car          &  0.999 (2.3e-04) &  \bestb{1.000} (3.8e-05) &  \bestb{1.000} (6.0e-05) &  \bestb{1.000} (0.0e+00) &  0.997(1.1e-03) & \bestf{0.998}(2.8e-04) & \bestf{0.998}(6.3e-04) & \bestf{0.998}(1.5e-04) & 0.993(8.7e-04) & 0.993(1.4e-04) \\
covtype      &  \bestb{0.999} (7.5e-06) &  \bestb{0.999} (3.6e-06) &  0.998 (2.5e-05) &  \bestb{0.999} (5.2e-06) &  0.997(1.4e-04) & \bestf{0.998}(1.4e-04) & 0.996(1.9e-04) & \bestf{0.998}(3.7e-05) & 0.974(7.6e-04) & 0.996(4.2e-05) \\
default-cb   &  0.780 (3.5e-04) &  \bestb{0.783} (1.2e-04) &  0.780 (3.9e-04) &  0.779 (6.7e-04) &  0.765(3.9e-03) & 0.775(7.4e-04) & 0.774(4.4e-04) & \bestf{0.778}(1.6e-03) & 0.772(9.0e-04) & 0.773(3.5e-04) \\
higgs        &  0.853 (6.7e-05) &  \bestb{0.857} (1.8e-05) &  0.847 (2.2e-05) & 0.853 (7.0e-05) &  0.820(9.9e-05) & \bestf{0.837}(4.5e-05) & 0.820(1.5e-04) & \bestf{0.837}(6.4e-05) & 0.813(1.1e-04) & 0.815(5.3e-05) \\
internet     &  0.934 (2.9e-04) &  0.910 (1.8e-04) &  \bestb{0.938} (1.3e-03) &  0.911 (1.1e-16) &  0.917(2.3e-03) & \bestf{0.935}(6.1e-04) & 0.926(1.5e-03) & 0.928(7.9e-04) & 0.925(4.5e-04) & 0.928(8.9e-04) \\
kddcup       &  \bestb{1.000} (6.1e-08) &   \bestb{1.000} (4.1e-07) &   \bestb{1.000} (5.8e-07) &    \bestb{1.000} (5.7e-07) &  0.997(5.2e-05) & 0.998(2.1e-03) & \bestf{1.000}(1.0e-05) & \bestf{1.000}(6.0e-06) & \bestf{1.000}(1.6e-05) & \bestf{1.000}(6.8e-06) \\
kick         &  \bestb{0.777} (7.7e-04) &  0.770 (2.8e-04) &  \bestb{0.777} (6.8e-04) &  0.771 (1.6e-03) &  0.749(1.1e-03) & \bestf{0.764}(7.9e-04) & 0.756(8.3e-04) & 0.763(5.4e-04) & 0.752(9.7e-04) & 0.754(5.2e-04)  \\
letter       &  \bestb{1.000} (1.3e-05) &  \bestb{1.000} (2.7e-06) &  \bestb{1.000} (4.1e-06) &  \bestb{1.000} (2.1e-05) &  0.997(2.6e-04) & \bestf{0.999}(1.2e-04) & 0.997(2.3e-04) & \bestf{0.999}(5.5e-05) & 0.995(4.8e-04) & 0.993(2.4e-04)  \\
satimage     &  \bestb{0.991} (2.1e-04) &  \bestb{0.991} (3.6e-05) &  \bestb{0.991} (2.3e-04) &  0.987 (0.0e+00) &  0.985(1.3e-03) & \bestf{0.991}(3.8e-04) & 0.986(1.2e-03) & \bestf{0.991}(1.7e-04) & 0.984(6.4e-04) & 0.983(1.9e-04) \\
sensorless   &  \bestb{1.000} (4.9e-07) &  \bestb{1.000} (1.4e-07) &  \bestb{1.000} (4.4e-06) &  \bestb{1.000} (2.9e-06) &  \bestf{1.000}(2.3e-05) & \bestf{1.000}(5.0e-07) & \bestf{1.000}(1.1e-04) & \bestf{1.000}(1.4e-05) & 0.999(5.4e-05) & \bestf{1.000}(7.3e-06) 
 \\
spambase     &  \bestb{0.990} (1.5e-04) &  \bestb{0.990} (5.2e-05) &  0.987 (1.2e-03) &  0.986 (0.0e+00) &  0.980(1.8e-03) & \bestf{0.987}(2.3e-04) & 0.983(1.0e-03) & \bestf{0.987}(2.2e-04) & 0.972(1.1e-03) & 0.974(2.9e-04) \\
\bottomrule
\end{tabular}}
\end{sideways}
\label{tab:performance-with-std}
\end{table}
\begin{table}[htbp]
\caption{The results of Table~\ref{tab:performance} measured through log-loss corresponding to all algorithms after hyperoptimization on the considered datasets. We observe that WW often improves the performance of RF or achieves a close result (smaller is better, bold is for best EGB performance, underline for best \texttt{RF}$n$, \texttt{WW}$n$ or \texttt{RD}$n$ performance).}
\resizebox{\textwidth}{!}{\begin{tabular}{l|cccc|cccccc}
\toprule
 &   \texttt{XGB} &  \texttt{LGBM} & \texttt{CB} & \texttt{HGB} & \texttt{RF}$10$ & \texttt{RF}$100$ & \texttt{WW}$10$ & \texttt{WW}$100$ & \texttt{RD}$10$ & \texttt{RD}$100$ \\
\midrule
adult        &  0.273 &  \bestb{0.271} &  0.281 &  0.275 &  0.302 & 0.294 & 0.296 & \bestf{0.293} & 0.299 & 0.296\\
bank         &  0.200 &  \bestb{0.196} &  0.208 &  0.202 &  0.214 &  0.207 &  0.206 &  \bestf{0.201} & 0.213 & 0.210\\
breastcancer &  0.126 &  \bestb{0.101} &  0.138 &  0.119 &  0.155 &  \bestf{0.122} & 0.135 &  0.126 & 0.164 & 0.176\\
car          &  0.048 &  0.040 &  0.051 &  \bestb{0.015} &  0.163 &  0.085 & \bestf{0.078} &  \bestf{0.078} & 0.194 & 0.190\\
covtype      &  0.087 &  \bestb{0.075} &  0.126 &  0.079 &  0.164 &  \bestf{0.118} &  0.141 &  0.123 & 0.495  & 0.233\\
default-cb   &  0.430 &  \bestb{0.429} &  0.430 &  0.431 &  0.439 &  \bestf{0.432} &  0.434 &  \bestf{0.432} & 0.435 & 0.435\\
higgs        &  0.472 &  \bestb{0.467} &  0.481 &  0.484 &  0.517 &  \bestf{0.499} &  0.518 &  \bestf{0.499} & 0.525 & 0.523\\
internet     &  1.446 &  1.552 &  \bestb{1.413} &  1.573 &  1.564 &  \bestf{1.450} &  1.505 &  1.520 & 1.561 & 1.560\\
kddcup       &  \bestb{0.000} &  \bestb{0.000} &  \bestb{0.000} &  \bestb{0.000} &  \bestf{0.000} &  \bestf{0.000} &  \bestf{0.000} & \bestf{0.000}  & 0.001  & 0.001\\
kick         &  \bestb{0.293} &  0.298 &  \bestb{0.293} &  0.295 &  0.306 &  \bestf{0.297} &  0.308 &  0.306 & 0.303 & 0.302\\
letter       &  0.116 &  \bestb{0.111} &  0.113 &  0.138 &  0.473 &  0.276 &  0.358 &  \bestf{0.274} & 0.691 & 0.807\\
satimage     &  \bestb{0.227} &  0.234 &  0.228 &  0.265 &  0.333 &  \bestf{0.261} &  0.313 &  0.265 & 0.372 & 0.365\\
sensorless   &  \bestb{0.004} &  \bestb{0.004} &  \bestb{0.004} &  0.005 &  0.033 &  \bestf{0.025} &  0.035 &  0.027 & 0.074 & 0.075\\
spambase     &  \bestb{0.115} &  0.121 &  0.136 &  0.137 &  0.203 &  0.216 &  0.178 &  \bestf{0.160}  & 0.215 & 0.210\\
\bottomrule
\end{tabular}}
\label{tab:performance-logloss}
\end{table}

\begin{table}[htbp]
\centering
\caption{The same training time table as in Table~\ref{tab:fit_time}, as average over 5 runs, with standard deviation computed from 5 runs reported between parenthesis, for default parameters for each model. Top: training time in seconds; bottom: test AUC. We observe that WW is generally among the fastest algorithms, for performances comparable to ones of all baselines (bold is for best EGB training time or performance, underline for best \texttt{RF} or \texttt{WW} training time or performance).}
\small
\begin{tabular}{l|cccc|cc}
\toprule
 & \multicolumn{6}{c}{Training time (seconds)} \\
 \cmidrule(lr){2-7} 
 &  \texttt{XGB} & \texttt{LGBM} & \texttt{CB} &  \texttt{HGB} & \texttt{RF} &  \texttt{WW} \\
\midrule
covtype  & 10 (0.6) & \bestb{3} (0.1) & 120 (9.3) &  14 (7.7) & 21 (0.9) & \bestf{3} (0.1) \\
higgs    &   36 (0.6) &  \bestb{30} (1.4) & 653 (8.7) &  85 (0.2) &  1389 (11.1) &  \bestf{179} (4.5) \\
internet &    9 (0.7) &   \bestb{4} (0.1) &    188 (2.4) &   8 (0.3) &  0.4 (0.0) &  \bestf{0.3} (0.0) \\
kddcup   &  175 (5.1) &  41 (2.6) &  2193 (13.2) &  \bestb{31} (0.2) &  208 (3.8) &  \bestf{12} (0.8) \\
kick     &    7 (0.2) &   \bestb{0.4} (0.0) &  50 (0.7) &  0.7 (0.1) & 31 (0.1) & \bestf{5} (0.0) \\
\bottomrule
\end{tabular}
\vspace{5pt}

\begin{tabular}{l|cccc|cc}
\toprule
& \multicolumn{6}{c}{Test AUC} \\
 \cmidrule(lr){2-7}
&  \texttt{XGB} & \texttt{LGBM} & \texttt{CB} &  \texttt{HGB} & \texttt{RF} &  \texttt{WW} \\
\midrule
covtype  & 0.986 (2e-04) &  0.978 (2e-03) &  \bestb{0.989} (9e-05) &  0.960 (1e-02) &  \bestf{0.998} (6e-05) &  0.979 (5e-04) \\
higgs    & 0.823 (3e-04) &  0.812 (2e-04) & \bestb{0.840} (8e-05) &  0.812 (2e-04) &  \bestf{0.838} (9e-05) &  0.813 (1e-04) \\
internet & \bestb{0.918} (2e-05) &  0.828 (0e+00) &  0.910 (8e-03) &  0.500 (0e+00) &  0.862 (3e-03) &  \bestf{0.889} (7e-03) \\
kddcup   & \bestb{1.000} (3e-07) &  0.638 (3e-02) &  0.988 (7e-03) &  0.740 (6e-02) &  0.998 (2e-03) &  \bestf{1.000} (3e-06) \\
kick     & 0.768 (4e-04) &  0.757 (0e+00) &  \bestb{0.781} (3e-04) &  0.773 (2e-03) &  0.747 (2e-03) &  \bestf{0.751} (2e-03) \\
\bottomrule
\end{tabular}
\label{tab:fit_time_std}
\end{table}

\section{Supplementary details about hyperparameter tuning}
\label{sec:exp-details}
In this Section, we provide extra information about hyperparameters optimization.
For XGBoost, LightGBM and CatBoost, with all other hyperparameters fixed, we use early stopping by monitoring the log loss on the validation set, the maximum number of boosting iterations being set at $5,000$. 
The best number of iterations is used together with other best hyperparameters to refit over the whole training set before evaluation on the test set.
For \texttt{scikit-learn}'s Random Forest and WildWood, we report results both for $10$ and $100$ trees, note that the default choice is $10$ for WildWood (since subtrees aggregation allows to use fewer trees than RF) while default is $100$ in \texttt{scikit-learn}.
We list the hyperparameters search space of each algorithm below.
\paragraph{XGBoost}
\begin{itemize}
    \item \texttt{eta}: log-uniform distribution $[e^{-7},1]$;
    \item \texttt{max\_depth}: discrete uniform distribution $[2,10]$;
    \item \texttt{subsample}: uniform $[0.5,1]$;
    \item \texttt{colsample\_bytree}: uniform $[0.5, 1]$;
    \item \texttt{colsample\_bylevel}: uniform $[0.5, 1]$;
    \item \texttt{min\_child\_weight}: log-uniform distribution $[e^{-16}, e^5]$;
    \item \texttt{alpha}: 0 with probability 0.5, and log-uniform distribution $[e^{-16}, e^2]$ with probability 0.5;
    \item \texttt{lambda}: 0 with probability 0.5, and log-uniform distribution $[e^{-16}, e^2]$ with probability 0.5;
    \item \texttt{gamma}: 0 with probability 0.5, and log-uniform distribution $[e^{-16}, e^2]$ with probability 0.5;
\end{itemize}
\paragraph{LightGBM}
\begin{itemize}
    \item \texttt{learning\_rate}: log-uniform distribution $[e^{-7},1]$;
    \item \texttt{num\_leaves}: discrete log-uniform distribution $[1,e^7]$;
    \item \texttt{feature\_fraction}: uniform $[0.5,1]$;
    \item \texttt{bagging\_fraction}: uniform $[0.5,1]$;
    \item \texttt{min\_data\_in\_leaf}: discrete log-uniform distribution $[1,e^6]$;
    \item \texttt{min\_sum\_hessian\_in\_leaf}: log-uniform distribution $[e^{-16}, e^5]$;
    \item \texttt{lambda\_l1}: 0 with probability 0.5, and log-uniform distribution $[e^{-16}, e^2]$ with probability 0.5;
    \item \texttt{lambda\_l2}: 0 with probability 0.5, and log-uniform distribution $[e^{-16}, e^2]$ with probability 0.5;
\end{itemize}
\paragraph{CatBoost}
\begin{itemize}
    \item \texttt{learning\_rate}: log-uniform distribution $[e^{-7},1]$;
    \item \texttt{random\_strength}: discrete uniform distribution over $\{1, 20\}$;
    \item \texttt{one\_hot\_max\_size}: discrete uniform distribution over $\{0, 25\}$;
    \item \texttt{l2\_leaf\_reg}: log-uniform distribution $[1,10]$;
    \item \texttt{bagging\_temperature}: uniform $[0,1]$.
\end{itemize}
\paragraph{HistGradientBoosting}
\begin{itemize}
    \item  \texttt{learning\_rate}: log-uniform distribution $[e^{-4},1]$;
    \item \texttt{max\_leaf\_nodes}: discrete log-uniform distribution $[1,e^7]$;
    \item \texttt{min\_samples\_leaf}: discrete log-uniform distribution $[1,e^6]$;
    \item \texttt{l2\_regularization}: 0 with probability 0.5, and log-uniform distribution $[e^{-16}, e^2]$ with probability 0.5;
\end{itemize}
\paragraph{RandomForest}
\begin{itemize}
    \item \texttt{max\_features}: uniform among \texttt{None}, \texttt{sqrt}, \texttt{log2}, \texttt{0.25}, \texttt{0.5} and \texttt{0.75};
    \item \texttt{max\_depth}: uniform among \texttt{None}, \texttt{sqrt} and \texttt{log2}, the latter two are meant in terms of the train sample size;
    \item \texttt{min\_samples\_leaf}: uniform over $\{1,5,10\}$ and we set $\mathtt{min\_samples\_split} = 2 \times \mathtt{min\_samples\_leaf}$;
\end{itemize}
\paragraph{WildWood}
\begin{itemize}
    \item \texttt{multiclass}: \texttt{multinomial} with probability 0.5, and \texttt{ovr} with probability 0.5;
    \item \texttt{min\_samples\_leaf}: uniform over $\{1,5,10\}$ and we set $\mathtt{min\_samples\_split} = 2 \times \mathtt{min\_samples\_leaf}$;
    \item \texttt{step}: log-uniform distribution $[e^{-3}, e^{6}]$;
    \item \texttt{dirichlet}: log-uniform distribution $[e^{-7}, e^{2}]$;
    \item \texttt{cat\_split\_strategy}:   \texttt{binary} with probability 0.5, and \texttt{all} with probability 0.5;
    \item \texttt{max\_features}: uniform among \texttt{None}, \texttt{sqrt}, \texttt{log2}, \texttt{0.25}, \texttt{0.5} and \texttt{0.75};
    \item \texttt{max\_depth}: uniform among \texttt{None}, \texttt{sqrt} and \texttt{log2}, the latter two are meant in terms of the train sample size;
    
\end{itemize}

\section{Datasets}
\label{sec:datasets}

The main characteristics of the datasets used in the paper are summarized in Table~\ref{tab:dataset-descrip}.
We provide in Table~\ref{tab:data-source} the URL of all the datasets used, most of them are from the UCI machine learning repository~\citep{Dua:2019}.
\begin{table}[htbp]
\centering
\begin{tabular}{lcccccc}
    \toprule
    Dataset & \# Samples & \# Features & \# Categorical features & \# Classes & Gini \\
    \midrule
    adult & 48,842 & 14 & 8 & 2 & 0.36 \\
    bank & 45,211 & 16 & 10 & 2 & 0.21 \\
    banknote & 1,372 & 4 & 0 & 2 & 0.49 \\
    breastcancer & 569 & 30 & 0 & 2 & 0.47 \\
    car & 1,728 & 6 & 6 & 4 & 0.46 \\
    covtype & 581,012 & 54 & 0 & 7 & 0.62  \\
    default\_cb & 30,000 & 23 & 3 & 2 & 0.34 \\
    higgs & 11,000,000 & 28 & 0 & 2 & 0.50  \\
    HCV & 615 & 13 & 2 & 2 & 0.19  \\
    internet & 10,108 & 70 & 70 & 46 & 0.88 \\
    ionosphere & 351 & 34 & 0 & 2 & 0.46 \\
    kddcup99 & 4,898,431 & 41 & 7 & 23 & 0.58 \\
    kick & 72,983 & 32 & 18 & 2 & 0.22 \\
    letter & 20,000 & 16 & 0 & 26 & 0.96 \\
    satimage & 5,104 & 36 & 0 & 6 & 0.81  \\
    sensorless & 58,509 & 48 & 0 & 11 & 0.91 \\
    spambase & 4,601 & 57 & 0 & 2 & 0.48 \\
    \bottomrule
\end{tabular}
\caption{Main characteristics of the datasets used in experiments, including number of samples, number of features, number of categorical features, number of classes and the Gini index of the class distribution on the whole datasets (rescaled between $0$ and $1$), in order to quantify class unbalancing.}
\label{tab:dataset-descrip}
\end{table}

\begin{table}[htbp]
    \centering
    \footnotesize
    \resizebox{\textwidth}{!}{\begin{tabular}{ll}
    \toprule
        Dataset & URL \\
        \midrule
        adult & \url{https://archive.ics.uci.edu/ml/datasets/Adult} \\
        bank & \url{https://archive.ics.uci.edu/ml/datasets/bank+marketing} \\
        banknote & \url{https://archive.ics.uci.edu/ml/datasets/banknote+authentication} \\
        breastcancer & \url{https://archive.ics.uci.edu/ml/datasets/breast+cancer+wisconsin+(diagnostic)}\\
        car & \url{https://archive.ics.uci.edu/ml/datasets/car+evaluation} \\
        covtype & \url{https://archive.ics.uci.edu/ml/datasets/covertype} \\
        default\_cb & \url{https://archive.ics.uci.edu/ml/datasets/default+of+credit+card+clients} \\
        HCV & \url{https://archive.ics.uci.edu/ml/datasets/HCV+data} \\
        higgs & \url{https://archive.ics.uci.edu/ml/datasets/HIGGS} \\
        internet & \url{https://kdd.ics.uci.edu/databases/internet_usage/internet_usage.html} \\
        ionosphere & \url{https://archive.ics.uci.edu/ml/datasets/ionosphere}\\
        kddcup99 & \url{https://kdd.ics.uci.edu/databases/kddcup99/kddcup99.html} \\
        kick & \url{https://www.openml.org/d/41162} \\
        letter & \url{https://archive.ics.uci.edu/ml/datasets/letter+recognition} \\
        satimage & \url{https://archive.ics.uci.edu/ml/datasets/Statlog+(Landsat+Satellite)} \\
        sensorless & \url{https://archive.ics.uci.edu/ml/datasets/dataset+for+sensorless+drive+diagnosis} \\
        spambase & \url{https://archive.ics.uci.edu/ml/datasets/spambase} \\
        \bottomrule
    \end{tabular}}
    \caption{The URLs of all the datasets used in the paper, giving direct download links and supplementary details.}
    \label{tab:data-source}
\end{table}
Note that the HCV data set labels were binarized by considering the class ``Blood Donor'' against all the others.

\section{Sensitivity of hyperparameters of Wildwood}
\label{apd:hyperparameters-sensitivity}

In Table~\ref{tab:hyperparameters-sensitivity} we illustrate the effects of hyperparameters on WW's performance on a few datasets, measured by the test AUC.
We can observe in this table that it is only weakly affected by different combinations of hyperparameters.
\begin{table}[htbp]
\centering
\footnotesize
\begin{tabular}{rrrr}
\toprule
\multicolumn{4}{c}{adult} \\
$n_{\text{min-leaf}}$ &  $\alpha$ &  $\eta$ &  AUC \\
\midrule
1 &       0.1 &  0.1 &    0.913 \\
1 &       0.1 &  1.0 &    0.916 \\
1 &       0.1 & 10.0 &    0.918 \\
1 &       0.5 &  0.1 &    0.913 \\
1 &       0.5 &  1.0 &    0.917 \\
1 &       0.5 & 10.0 &    0.919 \\
1 &       2.5 &  0.1 &    0.913 \\
1 &       2.5 &  1.0 &    0.917 \\
1 &       2.5 & 10.0 &    0.919 \\
5 &       0.1 &  0.1 &    0.913 \\
5 &       0.1 &  1.0 &    0.916 \\
5 &       0.1 & 10.0 &    0.918 \\
5 &       0.5 &  0.1 &    0.913 \\
5 &       0.5 &  1.0 &    0.916 \\
5 &       0.5 & 10.0 &    0.918 \\
5 &       2.5 &  0.1 &    0.913 \\
5 &       2.5 &  1.0 &    0.917 \\
5 &       2.5 & 10.0 &    0.918 \\
10 &       0.1 &  0.1 &    0.913 \\
10 &       0.1 &  1.0 &    0.916 \\
10 &       0.1 & 10.0 &    0.917 \\
10 &       0.5 &  0.1 &    0.913 \\
10 &       0.5 &  1.0 &    0.916 \\
10 &       0.5 & 10.0 &    0.917 \\
10 &       2.5 &  0.1 &    0.913 \\
10 &       2.5 &  1.0 &    0.916 \\
10 &       2.5 & 10.0 &    0.918 \\
\bottomrule
\end{tabular}
\hfill
\begin{tabular}{rrrr}
\toprule
\multicolumn{4}{c}{bank} \\
$n_{\text{min-leaf}}$ &  $\alpha$ &  $\eta$ &  AUC \\
\midrule
1 &       0.1 &  0.1 &    0.919 \\
1 &       0.1 &  1.0 &    0.926 \\
1 &       0.1 & 10.0 &    0.929 \\
1 &       0.5 &  0.1 &    0.920 \\
1 &       0.5 &  1.0 &    0.927 \\
1 &       0.5 & 10.0 &    0.929 \\
1 &       2.5 &  0.1 &    0.921 \\
1 &       2.5 &  1.0 &    0.927 \\
1 &       2.5 & 10.0 &    0.929 \\
5 &       0.1 &  0.1 &    0.919 \\
5 &       0.1 &  1.0 &    0.926 \\
5 &       0.1 & 10.0 &    0.928 \\
5 &       0.5 &  0.1 &    0.920 \\
5 &       0.5 &  1.0 &    0.926 \\
5 &       0.5 & 10.0 &    0.928 \\
5 &       2.5 &  0.1 &    0.921 \\
5 &       2.5 &  1.0 &    0.927 \\
5 &       2.5 & 10.0 &    0.928 \\
10 &       0.1 &  0.1 &    0.920 \\
10 &       0.1 &  1.0 &    0.926 \\
10 &       0.1 & 10.0 &    0.927 \\
10 &       0.5 &  0.1 &    0.920 \\
10 &       0.5 &  1.0 &    0.926 \\
10 &       0.5 & 10.0 &    0.927 \\
10 &       2.5 &  0.1 &    0.920 \\
10 &       2.5 &  1.0 &    0.926 \\
10 &       2.5 & 10.0 &    0.928 \\
\bottomrule
\end{tabular}
\hfill
\begin{tabular}{rrrr}
\toprule
\multicolumn{4}{c}{car} \\
$n_{\text{min-leaf}}$ &  $\alpha$ &  $\eta$ &  AUC \\
\midrule
                1 &       0.1 &  0.1 &    0.992 \\
                1 &       0.1 &  1.0 &    0.995 \\
                1 &       0.1 & 10.0 &    0.995 \\
                1 &       0.5 &  0.1 &    0.992 \\
                1 &       0.5 &  1.0 &    0.994 \\
                1 &       0.5 & 10.0 &    0.995 \\
                1 &       2.5 &  0.1 &    0.990 \\
                1 &       2.5 &  1.0 &    0.992 \\
                1 &       2.5 & 10.0 &    0.993 \\
                5 &       0.1 &  0.1 &    0.990 \\
                5 &       0.1 &  1.0 &    0.992 \\
                5 &       0.1 & 10.0 &    0.992 \\
                5 &       0.5 &  0.1 &    0.990 \\
                5 &       0.5 &  1.0 &    0.992 \\
                5 &       0.5 & 10.0 &    0.992 \\
                5 &       2.5 &  0.1 &    0.987 \\
                5 &       2.5 &  1.0 &    0.991 \\
                5 &       2.5 & 10.0 &    0.991 \\
               10 &       0.1 &  0.1 &    0.983 \\
               10 &       0.1 &  1.0 &    0.987 \\
               10 &       0.1 & 10.0 &    0.987 \\
               10 &       0.5 &  0.1 &    0.983 \\
               10 &       0.5 &  1.0 &    0.987 \\
               10 &       0.5 & 10.0 &    0.987 \\
               10 &       2.5 &  0.1 &    0.981 \\
               10 &       2.5 &  1.0 &    0.985 \\
               10 &       2.5 & 10.0 &    0.986 \\
\bottomrule
\end{tabular}
\caption{Areas under the ROC curves (AUC) obtained on test samples with WildWood (using 100 trees in the forest) on the adult, bank and car datasets with combinations of several hyper parameters. We observe that WildWood's performance does not vary significantly with respect to these hyperparameters.}
\label{tab:hyperparameters-sensitivity}
\end{table}

\section{Supplementary details about assets used (versions and licenses)}
\label{sec:details-assets-licences}

The versions and licenses of the libraries used in our experiments are:
\begin{itemize}
    \item \texttt{catboost} (0.25.1), Apache License 2.0
    \item \texttt{hyperopt} (0.2.5), license: \url{https://github.com/hyperopt/hyperopt/blob/master/LICENSE.txt}
    \item \texttt{joblib} (0.17), BSD-3-Clause License
    \item \texttt{lightgbm} (3.2.1), MIT License
    \item \texttt{matplotlib} (3.3.1), license:  \url{https://github.com/matplotlib/matplotlib/blob/master/LICENSE/LICENSE}
    \item \texttt{numba} (0.52), BSD 2-Clause "Simplified" License
    \item \texttt{numpy} (1.19.2), BSD-3-Clause License
    \item \texttt{pandas} (1.2.4), BSD-3-Clause License
    \item \texttt{python} (3.7.9), Python Software Fundation Licence version 2
    \item \texttt{scikit-learn} (0.24.2), BSD 3-Clause License
    \item \texttt{scipy} (1.5.4), BSD 3-Clause License
    \item \texttt{seaborn} (0.11), BSD-3-Clause License
    \item \texttt{xgboost} (1.4.1), Apache License 2.0
\end{itemize}

All the datasets used are publicly accessible and have no copyright restrictions.

\end{document}